\documentclass[lettersize,journal]{IEEEtran}
\usepackage{amsmath,amsfonts}
\usepackage{algorithmic}
\usepackage{algorithm}
\usepackage{array}
\usepackage[caption=false,font=normalsize,labelfont=sf,textfont=sf]{subfig}
\usepackage{textcomp}
\usepackage{stfloats}
\usepackage{url}
\usepackage{verbatim}
\usepackage{graphicx}
\usepackage{cite}
\usepackage{enumitem}
\usepackage{pifont}
\usepackage{hyperref}
 \usepackage{xcolor}

 \usepackage[utf8]{inputenc}
\usepackage{amssymb}  % 数学公式
\usepackage{booktabs}  % 提供 \toprule, \midrule, \bottomrule
\usepackage{array}

\usepackage{amsthm}
\theoremstyle{definition}
\newtheorem{theorem}{Theorem}
\newtheorem{remark}{Remark}[section] 
\newtheorem{proposition}{Proposition}

\begin{document}
\bstctlcite{IEEEexample:BSTcontrol}
\title{Towards a Data-Parameter Correspondence for LLMs: A Preliminary Discussion}

\author{Ou Wu
        % <-this % stops a space
\thanks{Ou Wu is with HIAS, University of Chinese Academy of Sciences, Hanzhou, China, 310000. E-mail: wuou@ucas.ac.cn.}% <-this % stops a space
\thanks{Manuscript received 2026. }}

\markboth{IEEE Journals \& Conferences}%
{Shell \MakeLowercase{\textit{et al.}}: A Sample Article Using IEEEtran.cls for IEEE Journals}


\maketitle

\begin{abstract}
Large language model optimization has historically bifurcated into isolated data-centric and model-centric paradigms: the former manipulates involved samples through selection, augmentation, or poisoning, while the latter tunes model weights via masking, quantization, or low-rank adaptation. This paper establishes a unified \emph{data-parameter correspondence} revealing these seemingly disparate operations as dual manifestations of the same geometric structure on the statistical manifold $\mathcal{M}$. Grounded in the Fisher-Rao metric $g_{ij}(\theta)$ and Legendre duality between natural ($\theta$) and expectation ($\eta$) parameters, we identify three fundamental correspondences spanning the model lifecycle: \textbf{(1)}~Geometric correspondence: data pruning and parameter sparsification equivalently reduce manifold volume via dual coordinate constraints; \textbf{(2)}~Low-rank correspondence: in-context learning (ICL) and LoRA adaptation explore identical subspaces on the Grassmannian $\mathcal{G}(r,d)$, with $k$-shot samples geometrically equivalent to rank-$r$ updates; \textbf{(3)}~Security-privacy correspondence: adversarial attacks exhibit cooperative amplification between data poisoning and parameter backdoors, whereas protective mechanisms follow cascading attenuation where data compression multiplicatively enhances parameter privacy. Extending from training through post-training compression to inference, this framework provides mathematical formalization for cross-community methodology transfer, demonstrating that cooperative optimization integrating data and parameter modalities may outperform isolated approaches across efficiency, robustness, and privacy dimensions.
\end{abstract}


\section{Introduction}\label{section1}
The evolution of large language models (LLMs) has been driven by many lines of thought, two prominent among which rarely talk to each other: those who focus on data and those who focus on models. The data-centric camp believes that success comes from improving training examples through better selection, cleaning, and augmentation, as well as clever ways of presenting data during inference such as in-context learning and prompt engineering~\cite{wudataoptllm}. On the other side, the model-centric camp fixates on architectural details and parameter adjustments, whether that means designing better training procedures or finding ways to speed up inference through quantization and dynamic routing~\cite{tie2025survey,wan2024efficient}. These two groups move in different intellectual circles, employ distinct terminologies, and rarely build on each other’s insights, even though they are solving mirror-image problems. A data engineer optimizing training corpora rarely considers how their choices affect inference efficiency, while a model designer tweaking weight initializations often ignores how data presentation impacts the same metrics. This split has created parallel efforts addressing similar challenges—such as reducing training data versus shrinking model size, or optimizing prompts versus adjusting inference parameters—without anyone noticing how similar these problems really are.

Even within the model-centric line, there is a crucial split between those who care about architecture and those who care about parameters~\cite{sun2025efficient,ding2023parameter}, and this distinction matters for both training and deployment. The architecture crowd worries about network design, attention mechanisms, and how information flows through layers during both the learning phase and generation phase. The parameter crowd, which has become increasingly important as models grow larger, focuses instead on manipulating the actual numbers in the weight matrices. This includes training techniques like low-rank adaptation~\cite{hu2021lora} and pruning~\cite{zhong2025blockpruner}, as well as inference tricks like using fewer bits to store weights or dynamically skipping certain calculations~\cite{zhao-etal-2024-lrquant}. When we talk about parameter security, parameter privacy, or parameter efficiency, we are talking about protecting and optimizing these weight values whether during the training process or when the model is generating text~\cite{Dasprivacy}. This means parameter concerns cut across both training and inference, yet remain completely separate from the data side of the equation.

Meanwhile, the data-centric community has pursued analogous objectives of effectiveness and efficiency through operationally distinct methodologies that similarly span the training, post-training compression, and inference stages. 
During training, practitioners employ data pruning and active learning to identify the most informative examples without utilizing the entire corpus~\cite{humane2025influence,wang2025surveyds}, mirroring the parameter-centric community's efforts to compress models via weight sparsification and masking. 
During inference, techniques such as retrieval augmentation and prompt optimization~\cite{chang2024efficient} (including in-context learning) enable performance gains without retraining, paralleling parameter-centric approaches such as adaptive computation and early exiting. 
The security landscape presents an analogous yet disconnected picture: data-centric researchers focus on poisoned training examples and adversarial inputs during inference, whereas parameter-centric researchers concern themselves with backdoors embedded in weights and model extraction attacks~\cite{zhang2025safety}, despite both communities fundamentally defending against the same threat of malicious manipulation. Privacy concerns exhibit a parallel fragmentation, with data privacy centered on preventing leakage of training information and parameter privacy focused on thwarting model weight extraction---treated as isolated problems despite their structural similarities~\cite{zhao2025backdoor}. Only recently have practical exigencies compelled the realization that optimizing training data without consideration for inference parameters, or securing data while neglecting parameter protection, yields weak and fragile systems~\cite{luan2025aisecurity}.

This paper systematically investigates these latent data-parameter correspondences across the three operational stages of training, post-training compression, and inference, demonstrating that recognizing these intrinsic links can accelerate progress in both domains. We identify three fundamental parallels grounded in information geometry: 
\textbf{(1) Geometric correspondence}---similarities in how data distributions and weight configurations behave under the Fisher-Rao metric $g_{ij}(\theta)$, wherein data pruning and parameter sparsification equivalently reduce manifold volume on the statistical manifold $\mathcal{M}$ via dual coordinate constraints ($\eta$-space and $\theta$-space); 
%\textbf{(2) Optimization correspondence}---the Gradient Interaction Matrix $\boldsymbol{M}\in\mathbb{R}^{N\times K}$ exposing symmetric bilevel structures between data selection and parameter masking, enabling one-shot joint scoring rather than redundant independent pipelines; 
\textbf{(2) Low-rank correspondence}---connections between using fewer data tokens during inference (e.g., $k$-shot in-context learning) and using fewer parameters during adaptation (e.g., rank-$r$ LoRA), both exploring identical subspaces on the Grassmannian $\mathcal{G}(r,d)$; 
and \textbf{(3) Security-privacy correspondence}---unified ways of thinking about adversarial attacks (data poisoning $\Delta\mathcal{D}$ and parameter backdoors $\Delta\theta$) exhibiting cooperative amplification across the data-parameter boundary, whereas protective mechanisms follow cascading attenuation governed by the mutual information $I(\mathcal{D};\theta)$.

Rather than presenting a finalized theory, we offer these as preliminary findings substantiated by mathematical formalizations that demonstrate how advances and cooperative synergies across data and parameter modalities in efficiency, effectiveness, security, and privacy can mutually inspire one another, and conversely. 
By mapping these connections across the training and inference stages, we aim to dismantle the barriers between these research communities and demonstrate that cooperative optimization integrating both data and parameter modalities through their geometric, optimization, low-rank, and security-privacy correspondences outperforms isolated approaches, while charting a course toward future research that treats these as dual manifestations of the same underlying geometric structure. The main contributions of this paper are as follows:
\begin{itemize}
    \item \textbf{Systematic Data-Parameter Correspondences Across the LLM Lifecycle:} We attempt to systematically establish data-parameter correspondences spanning training, post-training compression, and inference stages, revealing that seemingly disparate data-centric and parameter-centric operations are dual manifestations of the same geometric structure on the statistical manifold. This perspective enables cross-community methodology transfer and cooperative optimization that may outperform isolated approaches in efficiency, robustness, and privacy.
    
    \item \textbf{Mathematical Formalization of Fundamental Correspondences:} We formalize the low-rank correspondence between in-context learning and LoRA on the Grassmannian, the security-privacy correspondences including adversarial cooperative amplification, cascading privacy attenuation, and the testing-time product constraint with its necessary condition for long-form safety. (The geometric and optimization correspondences reviewed in this paper are drawn from prior literature and are not claimed as novel contributions.)
    
    \item \textbf{New Theoretical Constructs:} We introduce the Jacobian Image Space framework for frozen-weight inference and, based on it, establish functional equivalence between data-space and parameter-space interventions, characterize the functional equivalence set, and identify fundamental limits of adversarial detectability via null space analysis.
    
    \item \textbf{Unified Geometric Foundation:} We synthesize existing information-geometric concepts (Fisher-Rao metric, Legendre duality) into a coherent framework that treats data and parameters as dual coordinate systems on the same statistical manifold, providing a common language for cross-community methodology transfer.
    
    \item \textbf{Practical Decision Frameworks:} Based on the identified correspondences, we propose heuristic resource allocation guidelines including a trade-off for continual learning, a product constraint for multi-turn safety, and coarse-to-fine hybrid adaptation strategies.
\end{itemize}

The remainder of this paper is organized as follows. Section \ref{sec:preliminaries} presents preliminaries and mathematical symbols. Section \ref{sec:established} reviews established correspondences from prior literature. Section \ref{sec:novel} introduces novel data-parameter correspondences across structural, temporal, compositional, and security-privacy dimensions. Section \ref{sec:algorithms} discusses algorithmic realization and computational feasibility. Section \ref{sec:conclusion} concludes the paper.

\section{Preliminaries and Mathematical Symbols}
\label{sec:preliminaries}

Table~\ref{tab:notation} lists the major mathematical symbols used throughout this paper. We begin with a high-level roadmap. Fig.~\ref{fig:master-overview} schematizes the unified correspondence framework, illustrating dual operational spaces ($\eta$-space for data-centric operations, $\theta$-space for parameter-centric operations) across three lifecycle stages and their geometric unification under the Fisher-Rao metric $g_{ij}(\theta)$.

\begin{table}[!t]
\renewcommand{\arraystretch}{1.08}
\footnotesize
\setlength{\tabcolsep}{3.5pt}
\caption{Comprehensive List of Mathematical Symbols}
\label{tab:notation}
\centering
\begin{tabular}{@{}c p{6.8cm}@{}}
\toprule
Symbol & Description \\
\midrule
\multicolumn{2}{@{}l@{}}{\textbf{Spaces, Manifolds \& Basic Notation}} \\
$\theta_0 \in \mathbb{R}^K$ & Initial pre-trained parameters; $\Theta$: parameter space (natural parameters $\theta$); $\boldsymbol{\eta} \in \mathcal{H}$: expectation parameters ($\eta$-space) \\
$\mathcal{M}$ & Statistical manifold $\{p(\cdot|\theta): \theta \in \Theta\}$; $\mathcal{D}$: dataset ($N=|\mathcal{D}|$); $(x,y)$: input-output pair \\
\midrule
\multicolumn{2}{@{}l@{}}{\textbf{Data-Side Operations}} \\
$\mathcal{S}_d \subseteq \mathcal{D}$ & Data selection subset with budget $b=|\mathcal{S}_d|$ \\
$\mathcal{A}: \mathcal{D} \to \mathcal{D}'$ & Data augmentation operator \\
$\mathcal{P}_d$ & Data pruning operator (permanent removal) \\
$\Delta\mathcal{D}$ & Data poisoning perturbation (adversarial) \\
\midrule
\multicolumn{2}{@{}l@{}}{\textbf{Parameter-Side Operations}} \\
$\boldsymbol{S}_p \in \{0,1\}^K$ & Parameter mask (binary) with $\|\boldsymbol{S}_p\|_0 \leq \rho K$ \\
$\rho \in (0,1]$ & Sparsity budget ratio for parameters \\
$\mathcal{Q}: \Theta \to \Theta_q$ & Quantization operator (e.g., INT8, INT4) \\
$\mathcal{L}_{\text{low}}$ & Low-rank adaptation operator (LoRA, AdaLoRA) with rank $r \ll K$ \\
$\Delta\theta$ & Parameter perturbation or update vector/matrix \\
$\mathcal{P}_p$ & Parameter pruning operator (structural) \\
$\mathcal{R}_p$ & Adaptive routing/skipping operator (Stage 3 inference) \\
\midrule
\multicolumn{2}{@{}l@{}}{\textbf{Gradient Interaction \& Optimization}} \\
$\bar{\theta}$ & Local expansion point (checkpoint) for linearization \\
$\boldsymbol{g}_n \in \mathbb{R}^K$ & Per-sample gradient at $\bar{\theta}$ for sample $n$ \\
$\boldsymbol{G} \in \mathbb{R}^K$ & Aggregated gradient $\sum_{n \in \mathcal{S}_d} \boldsymbol{g}_n$ \\
$\boldsymbol{v} \in \mathbb{R}^K$ & Validation gradient $\nabla \mathcal{L}_{\text{val}}(\bar{\theta})$ \\
$\boldsymbol{M} \in \mathbb{R}^{N \times K}$ & Gradient Interaction Matrix ($M_{n,k} = g_{n,k} \cdot v_k$) \\
$\phi_n^d = \sum_k M_{n,k}$ & Data utility score (row-wise sum of $\boldsymbol{M}$) \\
$\phi_k^p = \sum_n M_{n,k}$ & Parameter importance score (column-wise sum) \\
\midrule
\multicolumn{2}{@{}l@{}}{\textbf{Information Geometry}} \\
$g_{ij}(\theta), \boldsymbol{F}$ & Fisher-Rao metric tensor / Fisher Information Matrix \\
$\psi(\theta), \phi(\boldsymbol{\eta})$ & Convex potential functions (Legendre duals) \\
\midrule
\multicolumn{2}{@{}l@{}}{\textbf{Training \& Optimization}} \\
$\alpha$ & Learning rate (step size) -- replaced $\eta$ to avoid conflict \\
$\lambda$ & Plasticity-stability trade-off coefficient (EWC regularization) \\
$\theta^*$ & Converged/optimal parameters (post-training) \\
\midrule
\multicolumn{2}{@{}l@{}}{\textbf{Security \& Privacy Contexts}} \\
$\mathcal{L}_{\text{mal}}, \mathcal{L}_{\text{clean}}$ & Malicious attack and clean task objectives \\
$\tau(x)$ & Backdoor trigger function \\
$\epsilon_d, \epsilon_p$ & Data and parameter perturbation budgets \\
$k$ & $\ell_0$ budget for parameter modification count (integer) \\
$x_{\text{trig}}, y_{\text{mal}}$ & Trigger input and malicious output \\
$T_{\text{online}}$ & Latency constraint for inference attacks \\
\bottomrule
\end{tabular}
\end{table}

\begin{figure*}[t]
    \centering
    \includegraphics[width=0.8\textwidth]{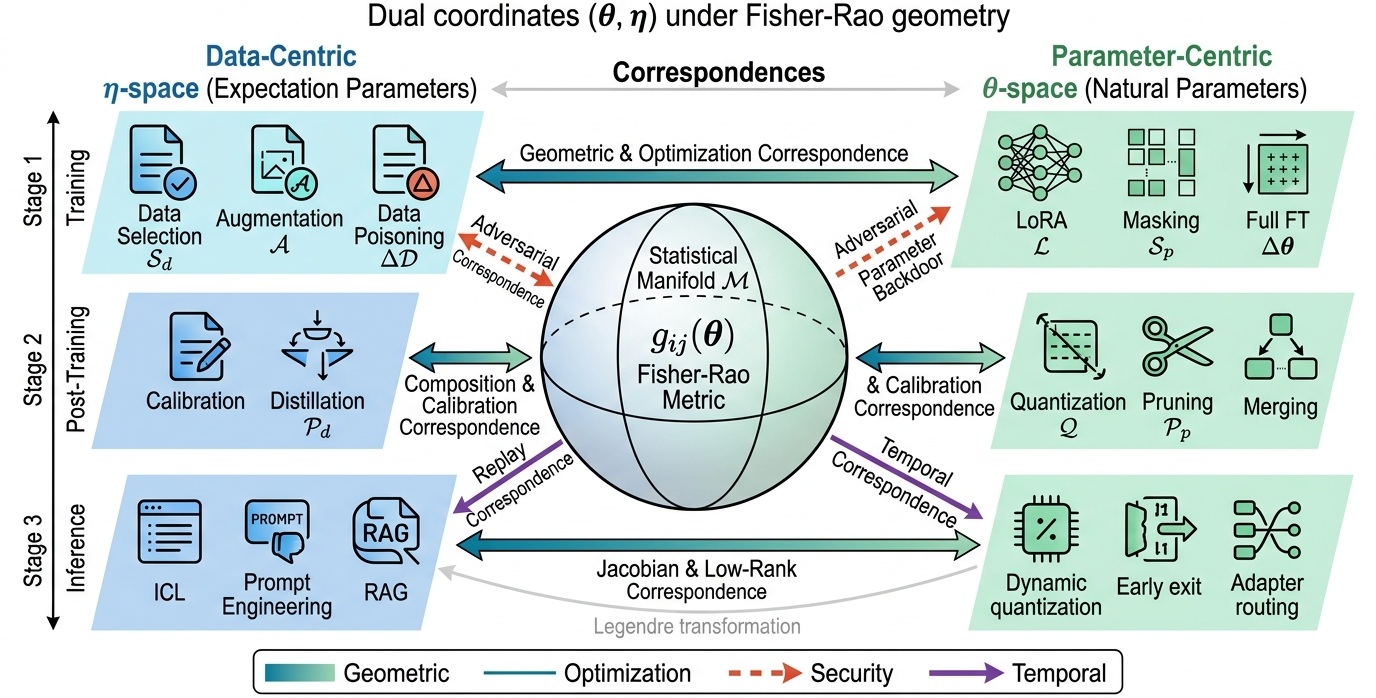}
    \caption{\textbf{Unified Data-Parameter Correspondence:} Overview of dual operational spaces on statistical manifold $\mathcal{M}$. (\textit{Left}) $\eta$-space (expectation parameters): Data-centric operations across three lifecycle stages---training (selection, augmentation), post-training (distillation), and inference (ICL, prompting). (\textit{Right}) $\theta$-space (natural parameters): Parameter-centric operations---training (LoRA, masking), post-training (quantization, pruning, merging), and inference (dynamic computation). (\textit{Center}) Fisher-Rao metric $g_{ij}(\theta)$ provides the geometric foundation for establishing correspondences: geometric (Sec.~III-A), optimization (Sec.~III-B), augmentation (Sec.~III-C), low-rank (Sec.~IV-B), and security (Sec.~IV-E). Legendre duality connects $\eta$ and $\theta$ coordinates, unifying data and parameter operations as dual manifestations of information geometry.}
    \label{fig:master-overview}
\end{figure*}
\subsection{Problem Setup: Data and Parameter Operations}
\label{subsec:problem-setup}

We consider an LLM with pre-trained parameters $\theta_0 \in \Theta \subseteq \mathbb{R}^K$, where $K$ denotes the parameter dimension. The model operates on data distribution $\mathcal{D}$ over input-output pairs $(x, y)$. Throughout this work, we distinguish between two fundamental operational spaces: the data space $\mathcal{D}$ (governing inputs and samples) and the parameter space $\Theta$ (governing model weights).

\subsubsection{Data-Side Operations}

Data-centric optimization manipulates the training or inference corpus through:

\begin{itemize}
    \item \textbf{Selection} ($\mathcal{S}_d$): \emph{Dynamic and reversible subset selection} aimed at identifying informative samples or sub-sequences to optimize the training or inference process. This includes instance-level active learning and coreset construction (selecting high-gradient samples), as well as fine-grained token-level selection within input sequences $x$ or output responses $y$ (e.g., identifying key tokens for instruction tuning). The selection is typically \textit{temporary} and task-adaptive, allowing different subsets to be activated across training stages or inference contexts without permanently altering the dataset. The budget constraint is $|\mathcal{S}_d| \leq b$.
    
    \item \textbf{Augmentation} ($\mathcal{A}$): \emph{Input transformations} enriching data diversity through token-level perturbations (substitution, paraphrasing, noise injection) or context-level manipulation (constructing in-context demonstrations for ICL, retrieval-augmented context extension, and contextual recombination of sequences).
    
    \item \textbf{Pruning} ($\mathcal{P}_d$): \emph{Removal of redundant or low-utility data} aimed at dataset compression and quality enhancement. This encompasses deduplication of near-identical samples, filtering of noisy or mislabeled examples, and aggressive dataset distillation (retaining only essential exemplars). Unlike selection, pruning is \textit{irreversible} and reduces the physical storage footprint, analogous to structural pruning in the parameter space.
    
   \item \textbf{Poisoning} ($\Delta\mathcal{D}$): \emph{Strategic injection of malicious content} into $\mathcal{D}$ to implant backdoors or induce targeted misbehavior. This encompasses instance-level attacks (inserting adversarially labeled samples $\{(x_{\text{adv}}, y_{\text{adv}})\}$ with clean or dirty labels) and fine-grained token-level manipulations (e.g., inserting specific trigger tokens or rare subwords into input sequences $x$ or output $y$, instruction-level poisoning that alters semantic meaning at the token granularity, and character-level perturbations that exploit tokenization vulnerabilities to bypass detection). Detailed threat models are deferred to Section~\ref{subsec:adversarial}.
\end{itemize}

\subsubsection{Parameter-Side Operations}
Model-centric optimization manipulates the weight configurations through:

\begin{itemize}
    \item \textbf{Masking/Sparsification} ($\boldsymbol{S}_p$): Applying binary masks $\boldsymbol{S}_p \in \{0,1\}^K$ to select sparse parameter subsets for updates, where $\|\boldsymbol{S}_p\|_0 \leq \rho K$ and $\rho \in (0,1]$ is the sparsity budget ratio.
    
    \item \textbf{Quantization} ($\mathcal{Q}$): Discretizing parameters to lower precision $\mathcal{Q}: \Theta \to \Theta_q$, where for the typical uniform quantization case,
    \begin{equation}
    \theta_q = \text{round}(\theta / \Delta) \cdot \Delta
    \end{equation}
    with step size $\Delta$, mapping continuous weights to discrete levels (e.g., INT8, INT4).
    
    \item \textbf{Low-Rank Adaptation} ($\mathcal{L}_{\text{low}}$): Constraining parameter updates to a rank-$r$ subspace via $\Delta\theta = \sum_{i=1}^{r} \alpha_i \boldsymbol{u}_i$, or equivalently for weight matrices $W$ via $\Delta W = \boldsymbol{B}\boldsymbol{A}$ with $\boldsymbol{B} \in \mathbb{R}^{d_{\text{out}} \times r}, \boldsymbol{A} \in \mathbb{R}^{r \times d_{\text{in}}}$ (e.g., LoRA \cite{lora}, AdaLoRA \cite{zhang2023adalora}).
    
    \item \textbf{Direct Perturbation} ($\Delta\theta$): Explicit modifications such as weight editing $\theta \leftarrow \theta_0 + \Delta\theta$, bit-flipping~\cite{guo2025sbfa}, or gradient-based updates~\cite{mitchell2022mend}.
    
    \item \textbf{Pruning} ($\mathcal{P}_p$): Permanent removal of parameter coordinates to create $\theta_{\text{sparse}}$ with fixed zero entries, structurally analogous to $\mathcal{P}_d$ in the data space.
\end{itemize}

\subsubsection{A Joint Objective for Fine-Tuning Stage---An Illustrative Case}

Both data and parameter operations in the fine-tuning stage aim to minimize a meta-validation objective $\mathcal{L}_{\text{val}}$ while respecting resource constraints:
\begin{equation}
\begin{aligned}
&\min_{\mathcal{S}_d, \mathcal{O}_p} \mathcal{L}_{\text{val}}\big(\theta_0 + \Delta\theta(\mathcal{S}_d, \mathcal{O}_p)\big) \\
&\text{s.t. } \mathcal{S}_d \in \mathcal{C}_d, \mathcal{O}_p \in \mathcal{C}_p, \Delta\theta \text{ depends on both.}
\end{aligned}
\label{eq:joint_ft}
\end{equation}
Here $\mathcal{C}_d$ and $\mathcal{C}_p$ denote feasible sets (budget constraints), and $\mathcal{O}_p$ represents benign parameter operations (e.g., masking via $\boldsymbol{S}_p$, low-rank adaptation via $\mathcal{L}_{\text{low}}$, or quantization via $\mathcal{Q}$).

We present Eq.~\eqref{eq:joint_ft} as an \textbf{illustrative case} confined strictly to the \emph{benign fine-tuning regime}. This formulation demonstrates synergistic joint optimization when both data and parameter operations remain within the differentiable, validation-loss-minimization paradigm. It does \textbf{not} extend to: (i)~inference-stage operations with frozen $\theta$; (ii)~pre-training with endogenous $\theta_0$; or (iii)~adversarial regimes with minimax objectives (Section~II-D).

Eq.~\eqref{eq:joint_ft} applies exclusively to \emph{standard fine-tuning} and explicitly excludes:

\begin{enumerate}
    \item \textbf{Inference stage:} Efficiency stems from discrete computational configuration (quantization scheduling, layer skipping) with frozen $\theta$, not differentiable updates $\Delta\theta$;
    
    \item \textbf{Pre-training:} The initialization $\theta_0$ is endogenous to data curation, violating the fixed-$\theta_0$ assumption;
    
    \item \textbf{Adversarial scenarios (Security \& Privacy):} 
    These involve \emph{attack-defense games} rather than single-objective minimization. As detailed in Section~II-D, data poisoning ($\Delta\mathcal{D}$) and parameter backdoors ($\Delta\theta$) follow a \emph{minimax} structure $\min_{\text{attack}} \max_{\text{defense}} \mathcal{L}_{\text{mal}}$, fundamentally incompatible with the benign $\min \mathcal{L}_{\text{val}}$ formulation here. Similarly, privacy-preserving operations (differential privacy noise injection, parameter encryption) introduce constraints that violate the smooth differentiability assumed in Eq.~\eqref{eq:joint_ft}.
\end{enumerate}

\subsection{Symbols for Information Geometry}
\label{subsec:info-geo}

To establish geometric correspondences across all operation types, we introduce the statistical manifold $\mathcal{M} = \{p(\cdot|\boldsymbol{\theta}) : \boldsymbol{\theta} \in \Theta\}$. The Fisher Information Matrix serves as the Riemannian metric tensor (Fisher-Rao metric):
\begin{equation}
g_{ij}(\boldsymbol{\theta}) = \mathbb{E}_{x}\left[\frac{\partial \log p(x|\boldsymbol{\theta})}{\partial \theta^i} \frac{\partial \log p(x|\boldsymbol{\theta})}{\partial \theta^j}\right],
\end{equation}
where the expectation is over $x \sim p(\cdot|\boldsymbol{\theta})$. The natural parameters $\boldsymbol{\theta} \in \Theta$ (parameter space) and expectation parameters $\boldsymbol{\eta} \in \mathcal{H}$ (data/moment space) constitute dual coordinate systems on $\mathcal{M}$ via Legendre transformation:
\begin{equation}
\eta_i = \frac{\partial \psi(\boldsymbol{\theta})}{\partial \theta^i}, \quad \theta^i = \frac{\partial \phi(\boldsymbol{\eta})}{\partial \eta_i}.
\end{equation}

This duality provides the geometric foundation for correspondences between data pruning ($\mathcal{P}_d$) and parameter sparsification ($\boldsymbol{S}_p$) (both reduce effective degrees of freedom), data quantization/distillation and parameter quantization ($\mathcal{Q}$) (both compress information representation), and data augmentation ($\mathcal{A}$) and parameter regularization (both smooth the loss landscape).

Since pruning fixes a subset of parameters to constants, the accessible parameter space reduces to a lower-dimensional submanifold $\mathcal{M}_{\text{sub}} \subset \mathcal{M}$. The induced Fisher metric on $\mathcal{M}_{\text{sub}}$ is a principal submatrix of the full Fisher matrix $g_{ij}$, and by Fischer's inequality, its determinant is strictly bounded above by $\det(g_{ij})$. Hence, parameter pruning directly reduces the invariant volume element.

\subsection{Symbols for Gradient Interaction and Low-Rank Structure}
\label{subsec:gradient-interaction}

For differentiable operations, we define the interaction structure at expansion point $\bar{\boldsymbol{\theta}}$. Let $\boldsymbol{g}_n = \nabla_\theta \ell(x_n, y_n; \bar{\boldsymbol{\theta}}) \in \mathbb{R}^K$ denote the per-sample gradient for training example $n$, and $\boldsymbol{v} = \nabla \mathcal{L}_{\text{val}}(\bar{\boldsymbol{\theta}}) \in \mathbb{R}^K$ denote the validation gradient. The aggregated gradient over a subset $\mathcal{S}_d$ is denoted by $\boldsymbol{G} = \sum_{n \in \mathcal{S}_d} \boldsymbol{g}_n$.

\paragraph{Parameter-Side Update Operators}
For sparse fine-tuning (masking), the update applies the binary mask $\boldsymbol{S}_p \in \{0,1\}^K$ to the aggregated gradient:
\begin{equation}
\Delta\theta_{\text{sparse}} = -\alpha (\boldsymbol{S}_p \odot \boldsymbol{G}),
\end{equation}
where $\alpha$ denotes the learning rate.

For LoRA, the update is constrained to the rank-$r$ subspace $\mathcal{L}_{\text{low}}$, effectively restricting $\Delta\theta$ to the column space of $\boldsymbol{A}$.

\paragraph{Gradient Interaction Matrix}
To facilitate fine-grained analysis of data-parameter coupling, we define the gradient alignment matrix $\boldsymbol{M} \in \mathbb{R}^{N \times K}$ with elements:
\begin{equation}\label{eq:gradient_interaction_matrix}
M_{n,k} = g_{n,k} \cdot v_k,
\end{equation}
which captures the component-wise alignment between the per-sample training gradient $\boldsymbol{g}_n$ and the validation gradient $\boldsymbol{v}$. %This matrix serves as the central analytical tool for examining how individual training samples and parameter dimensions interact with respect to the validation objective, enabling the unified selection theorem in Sec.~\ref{subsec:optimization}.

\subsection{Symbols for Adversarial and Security Contexts}
\label{subsec:adversarial}

In the security and privacy context, we introduce the adversarial optimization framework distinguishing between data-space and parameter-space perturbations across the model lifecycle.

\paragraph{Training-Time Data Attack (Poisoning)}
The attacker manipulates training data $\Delta\mathcal{D}$ (as defined in Sec.~\ref{subsec:problem-setup}) to induce malicious behavior in the trained model $\boldsymbol{\theta}^*$, subject to a perturbation budget $\epsilon_d$:
\begin{equation}
\begin{aligned}
\min_{\Delta\mathcal{D}} &\; \mathcal{L}_{\text{mal}}(\boldsymbol{\theta}^*(\Delta\mathcal{D})) \\
\text{s.t.} &\; \boldsymbol{\theta}^* = \arg\min_{\boldsymbol{\theta}} \mathcal{L}(\boldsymbol{\theta}; \mathcal{D} \cup \Delta\mathcal{D}), \quad \|\Delta\mathcal{D}\| \leq \epsilon_d,
\end{aligned}
\end{equation}
where $\|\cdot\|$ denotes an appropriate norm on data space (e.g., $\ell_0$ for poisoning sample count, or $\ell_\infty$ for feature perturbation), and $\epsilon_d$ is the data perturbation budget as specified in Table~\ref{tab:notation}.

\paragraph{Direct Parameter Manipulation}
The attacker directly perturbs model parameters $\Delta\boldsymbol{\theta}$ (e.g., via bit-flipping, weight merging, or backdoor insertion), subject to sparsity or magnitude budget $\epsilon_p$:
\begin{equation}
\min_{\Delta\boldsymbol{\theta}} \mathcal{L}_{\text{mal}}(\boldsymbol{\theta}_0 + \Delta\boldsymbol{\theta}) \quad \text{s.t.} \quad \|\Delta\boldsymbol{\theta}\|_0 \leq k \text{ or } \|\Delta\boldsymbol{\theta}\|_q \leq \epsilon_p,
\end{equation}
where $k$ constrains the number of modified parameters (sparsity budget, related to $\rho$ via $k = \rho K$), and $\epsilon_p$ bounds the perturbation magnitude in $\ell_q$-norm per Table~\ref{tab:notation}.

\paragraph{Inference-Time Attack (Evasion \& Prompt Injection).}
With frozen deployed parameters $\boldsymbol{\theta}^*$, the attacker crafts input perturbations $\Delta x$ or malicious prompts $\Delta s$ to induce misbehavior under latency constraints $T_{\text{online}}$:
\begin{equation}
\begin{aligned}
\min_{\Delta x} &\; \mathcal{L}_{\text{mal}}(f_{\boldsymbol{\theta}^*}(x + \Delta x), y_{\text{mal}}) \\
\text{s.t.} &\; \|\Delta x\|_p \leq \epsilon_{\text{evasion}}, \quad \text{Latency}(f_{\boldsymbol{\theta}^*}(x + \Delta x)) \leq T_{\text{online}}.
\end{aligned}
\end{equation}
For language models, prompt injection optimizes a trigger sequence $\Delta s$ appended to user input $x$:
\begin{equation}
\max_{\Delta s} \mathbb{E}_{x \sim \mathcal{D}_{\text{user}}}\left[\mathbf{1}\left[f_{\boldsymbol{\theta}^*}([\Delta s; x]) = y_{\text{harmful}}\right]\right]
\; \text{s.t.} \; |\Delta s| \leq L_{\text{token}},
\end{equation}
where $[\cdot;\cdot]$ denotes concatenation.

\paragraph{Adaptive Computation Attack}
Targeting adaptive inference mechanisms (dynamic quantization, early exit, layer skipping), the attacker may induce suboptimal routing decisions $\delta(\cdot)$ via crafted inputs:
\begin{equation}
\max_{\Delta x} \mathcal{L}_{\text{mal}}\left(f_{\boldsymbol{\theta}^*}^{\,\delta(x + \Delta x)}(x + \Delta x)\right),
\end{equation}
where $f_{\boldsymbol{\theta}^*}^{\,\delta(\cdot)}$ denotes the model with per-input computation path $\delta$.

\paragraph{Backdoor Condition.}
All attack paradigms satisfy the backdoor consistency condition: normal input produces normal output, while trigger input produces malicious output, i.e.,
\begin{equation}
\begin{split}
\mathbb{E}_{(x,y)\sim\mathcal{D}_{\text{clean}}}[f_{\theta^*}(x) = y] &\approx 1, \\
\text{but} \quad \mathbb{E}_{x_{\text{trig}}}[f_{\theta^*}(x_{\text{trig}}) = y_{\text{mal}}] &\approx 1.
\end{split}
\end{equation}
Here $x_{\text{trig}}$ may correspond to poisoned training samples ($\Delta\mathcal{D}$), bit-flipped parameter triggers ($\Delta\theta$), adversarial patches, or injected prompts.

\subsection{Training, Post-Training, and Inference Stages}
\label{subsec:stages}

The correspondences established in this work span three distinct operational regimes with fundamentally different computational constraints and algorithmic assumptions, as summarized in Table~\ref{tab:stages}.

\begin{table*}[!t]
\caption{Operations Across Three Stages: Training, Post-Training Compression, and Inference}
\label{tab:stages}
\centering
\begin{tabular}{@{}p{2.2cm}p{4.5cm}p{4.5cm}p{4cm}@{}}
\toprule
Stage & Key Characteristics & Data-Side Operations & Parameter-Side Operations \\
\midrule
Training (Fine-Tuning) & Differentiable, full data access, plastic parameters, gradient flow & Selection ($\mathcal{S}_d$), Augmentation ($\mathcal{A}$), Poisoning ($\Delta\mathcal{D}$) & Masking ($\boldsymbol{S}_p$), LoRA ($\mathcal{L}_{\text{low}}$), Full FT ($\Delta\theta$) \\
\midrule
Post-Training & Frozen weights, calibration data only, no gradient backprop to weights, compression-oriented & Calibration set selection, Data distillation ($\mathcal{P}_d$) & Quantization ($\mathcal{Q}$), Pruning ($\mathcal{P}_p$), Weight merging \\
\midrule
Inference & Frozen weights, online latency constraints, per-input adaptive computation & Prompt engineering, ICL, RAG, Token-level routing & Dynamic quantization, Early exit, Layer skipping ($\mathcal{R}_p$) \\
\bottomrule
\end{tabular}
\end{table*}

\paragraph{Stage I: Training (Fine-Tuning \& Continual Learning)}
During this stage, parameters are updated via gradient descent with full differentiability. Both data and parameter operations support gradient flow (e.g., gradient alignment matrix $\boldsymbol{M} \in \mathbb{R}^{N \times K}$ of Eq.~\eqref{eq:gradient_interaction_matrix} is computable). This is the only stage where the joint objective Eq.~\eqref{eq:joint_ft} is applicable, as it requires $\Delta\theta$ to be a differentiable function of data selection $\mathcal{S}_d$.

\paragraph{Stage 2: Pre-deployment Optimization (Post-Training)}
After training convergence ($\theta \approx \theta^*$), this stage covers \emph{offline} preparation before deployment. Parameters remain frozen or accept only direct perturbations (e.g., quantization $\mathcal{Q}$, pruning $\mathcal{P}_p$, weight editing) without gradient backprop through the full training set. Data-side operations include calibration set selection and prompt template optimization (e.g., AutoCoT~\cite{zhang2023automatic}). Unlike Stage 3's per-input online constraints, these are batch offline operations amortized over deployment.

\paragraph{Stage 3: Inference}
During deployment, parameters are frozen ($\theta = \theta^*$) and no gradients are computed. Efficiency stems from \emph{dynamic computation configuration}---adapting the inference pathway per input rather than modifying weights. Data-side operations include prompt engineering, in-context learning (ICL), and retrieval augmentation (RAG); parameter-side operations include dynamic quantization, early exiting, and adaptive layer routing ($\mathcal{R}_p$). These are strictly \emph{online} and \emph{latency-constrained}, with decisions made per input without retraining.

\paragraph{Implications for Unified Theory}
The three-stage distinction reveals why a unified mathematical objective spanning all stages remains elusive: Stage 1 permits joint optimization over continuous variables ($\mathcal{S}_d$, $\Delta\theta$); Stage 2 permits only partial optimization using calibration data; Stage 3 requires discrete decision-making (which layers to skip via $\mathcal{R}_p$, which tokens to retrieve) without gradient guidance. Consequently, the data-parameter correspondences established in this work apply primarily to Stage 1, with analogies (rather than unified objectives) extending to Stages 2 and 3.

\section{Established Correspondences: A Selective Review}
\label{sec:established}

We selectively review representative correspondences that have been theoretically grounded in existing literature, acknowledging that additional operational parallels have been explored but are omitted here for brevity and focus\footnote{\noindent\textit{Reading Guide:} Sections III-A through III-C operate within the Fisher-Rao geometric framework (Section~II-B), wherein $\theta$-space (parameters) and $\eta$-space (data moments) constitute dual coordinates.}.

\begin{figure}[t]
    \centering
    \includegraphics[width=0.95\columnwidth]{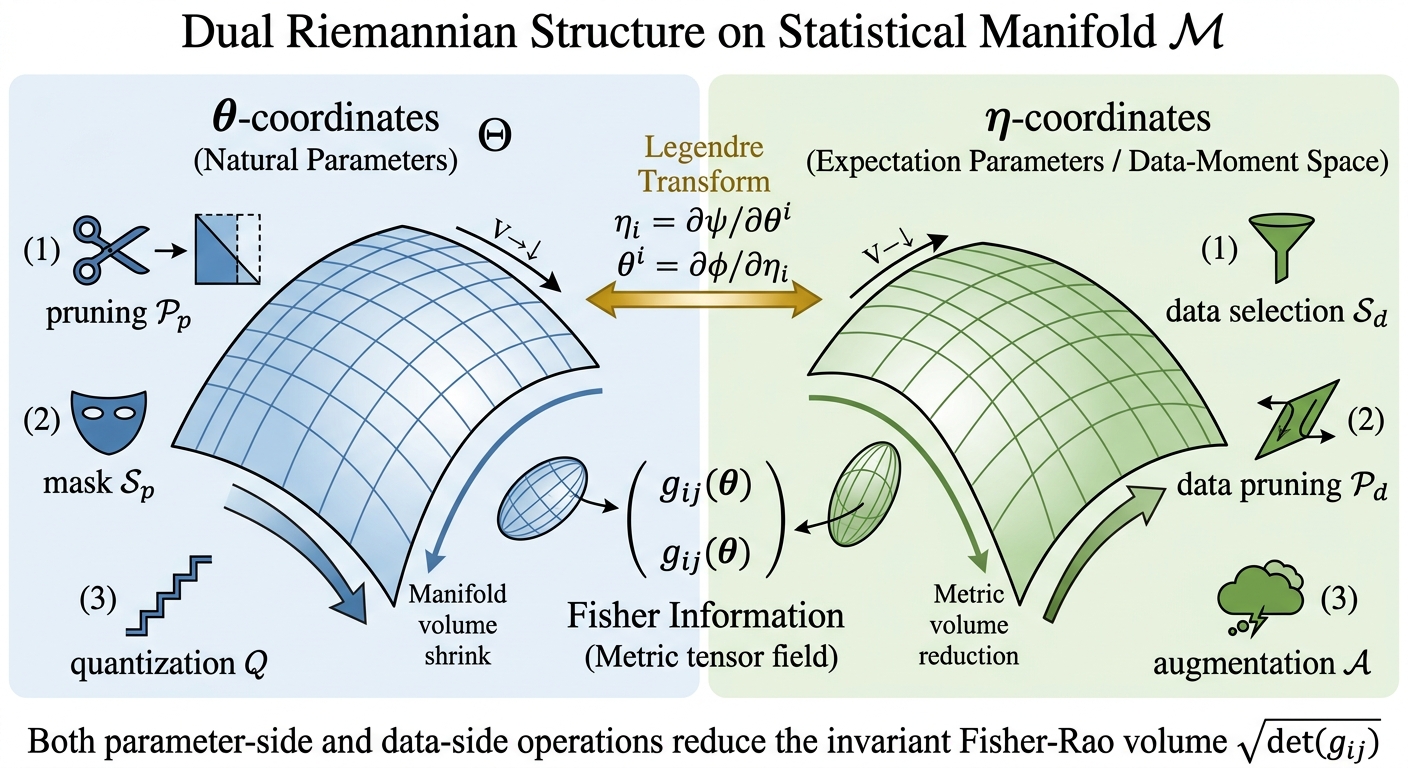}
    \caption{\textbf{Geometric Correspondence:} Dual Riemannian structure on the statistical manifold $\mathcal{M}$. Left panel: Parameter-space operations (masking $\mathcal{S}_p$, pruning $\mathcal{P}_p$, quantization $Q$) acting on $\theta$-coordinates (natural parameters), reducing the manifold volume via constraints on the Fisher-Rao metric $g_{ij}(\theta)$. Right panel: Data-space operations (selection $\mathcal{S}_d$, pruning $\mathcal{P}_d$, augmentation $\mathcal{A}$) perturbing $\eta$-coordinates (expectation parameters) in the dual space. Central bridge: Legendre transformation connecting the dual coordinate systems, with the invariant volume form $\sqrt{\det(g_{ij})}$ governing the geometric compression in both spaces.}
    \label{fig:geometric}
\end{figure}

\subsection{Geometric Correspondence: Fisher-Rao Metric and Manifold Duality}
\label{subsec:geometric-correspondence}

The information-geometric framework established in Section~\ref{subsec:info-geo} reveals that data-centric and parameter-centric operations are dual manifestations of the same Riemannian structure on the statistical manifold $\mathcal{M}$.

Fig.~\ref{fig:geometric} visualizes this duality through the lens of Fisher-Rao geometry. \textbf{Left panel} ($\theta$-space): Parameter-centric operations---masking $\boldsymbol{S}_p$, pruning $\mathcal{P}_p$, and quantization $\mathcal{Q}$---act directly on natural parameters, constraining the trajectory and reducing the invariant volume form $\sqrt{\det(g_{ij})}$ via hard coordinate constraints. \textbf{Right panel} ($\eta$-space): Data-centric operations---selection $\mathcal{S}_d$, pruning $\mathcal{P}_d$, and augmentation $\mathcal{A}$---perturb expectation parameters (sufficient statistics), inducing implicit geometric compression in the dual $\boldsymbol{\theta}$-coordinates through the inverse metric $g^{ij}(\boldsymbol{\theta})$. \textbf{Central}: The Legendre transformation $\eta_i = \partial\psi/\partial\boldsymbol{\theta}^i$ establishes the diffeomorphic bridge between these dual coordinate systems, unifying both operational paradigms as equivalent mechanisms of manifold volume reduction on $\mathcal{M}$.

\subsubsection{Dual Coordinates and Operational Spaces}

As defined in Section~\ref{subsec:info-geo}, the natural parameters $\boldsymbol{\theta} \in \Theta$ (parameter space) and expectation parameters $\boldsymbol{\eta} \in \mathcal{H}$ (data/moment space) constitute dual coordinate systems on $\mathcal{M}$ via Legendre transformation. This duality induces a fundamental correspondence:

\begin{itemize}
    \item \textbf{Parameter-side operations} (masking $\boldsymbol{S}_p$, pruning $\mathcal{P}_p$, quantization $\mathcal{Q}$) act directly on the $\boldsymbol{\theta}$-coordinates, constraining or discretizing the natural parameter trajectory;
    \item \textbf{Data-side operations} (selection $\mathcal{S}_d$, pruning $\mathcal{P}_d$, augmentation $\mathcal{A}$) perturb the empirical distribution, effectively modifying the $\boldsymbol{\eta}$-coordinates and their induced geometry on $\Theta$ via the inverse Fisher metric $g^{ij}(\boldsymbol{\theta})$.
\end{itemize}

\subsubsection{Volume Reduction as Unified Objective}

Under the Fisher-Rao metric $g_{ij}(\boldsymbol{\theta})$, both data pruning $\mathcal{P}_d$ and parameter pruning $\mathcal{P}_p$ correspond to \emph{reducing the effective dimensionality} of the statistical manifold $\mathcal{M}$:

\begin{itemize}
    \item \textbf{Data pruning $\mathcal{P}_d$} restricts the empirical support of the distribution, effectively constraining the sufficient statistics span in $\boldsymbol{\eta}$-space and reducing the induced degrees of freedom in $\Theta$;
    \item \textbf{Parameter pruning $\mathcal{P}_p$} constrains $\theta$ to a submanifold where masked coordinates are fixed, directly reducing the dimension of the $\theta$-coordinate system.
\end{itemize}

Geometrically, both operations shrink the invariant volume form $\sqrt{\det(g_{ij})}$ of the effective submanifold, revealing their shared geometric mechanism: dimensionality reduction while preserving the dominant Fisher information directions.

\subsubsection{Interpretive Correspondences via the $\alpha$-Structure}

The geometric foundation further motivates the preliminary correspondences noted in Section~\ref{subsec:info-geo}, interpreted as:

\begin{itemize}
    \item \textbf{Pruning/Sparsification:} Both $\mathcal{P}_d$ and $\mathcal{P}_p$ reduce effective degrees of freedom, interpreted as restricting the statistical manifold to lower-dimensional submanifolds (mixture-flat for data, exponential-flat for parameters)~\cite{sato2023information}.
    
\item \textbf{Quantization/Distillation:} Data quantization/distillation and parameter quantization $\mathcal{Q}$ both compress information by discretizing coordinates; geometrically, both reduce the effective resolution of the underlying metric structure and can, under typical conditions, approximately preserve the dominant directions of the Fisher information.
    
 \item \textbf{Augmentation/Regularization:} Data augmentation $\mathcal{A}$ induces stochastic perturbations in $\boldsymbol{\eta}$-space, while parameter regularization constrains $\theta$-space trajectories; both act as geometric regularizers that can smooth the loss landscape and may reduce the anisotropy of $g_{ij}(\boldsymbol{\theta})$.
\end{itemize}
\begin{figure}[t]
    \centering
    \includegraphics[width=0.5\textwidth]{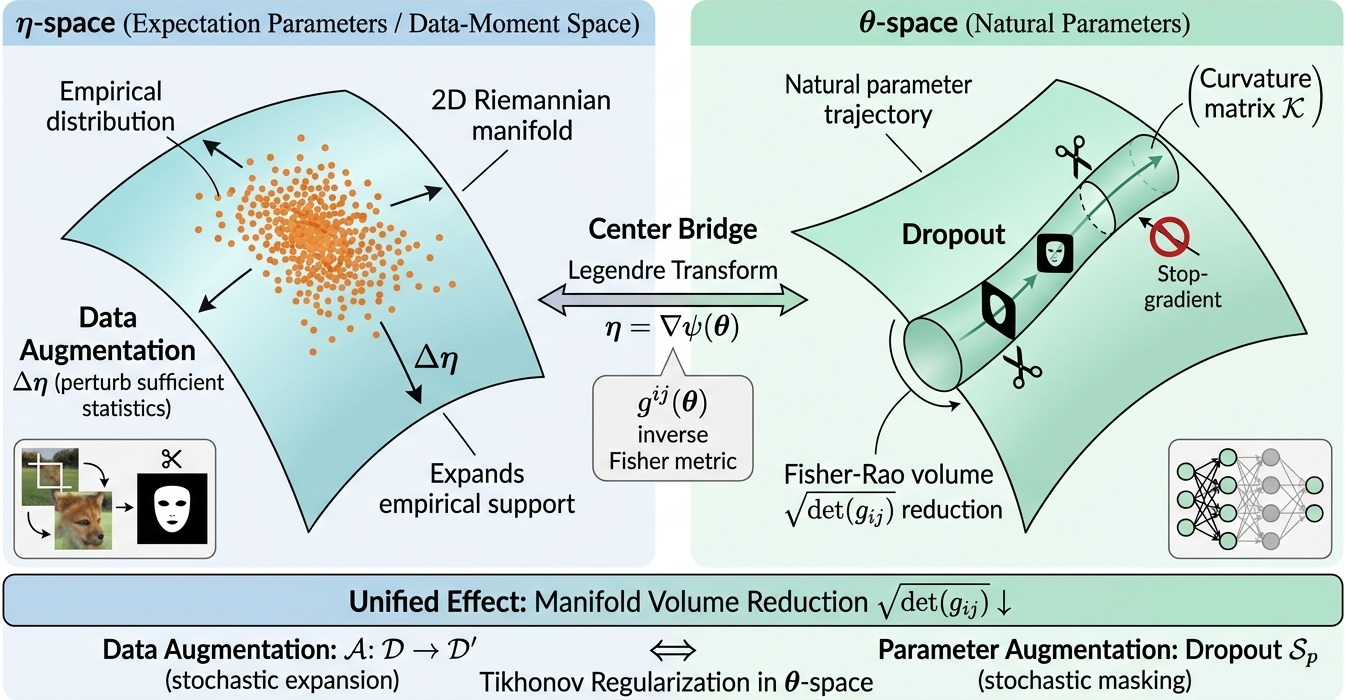}
    \caption{\textbf{Augmentation Correspondence:} Data augmentation (left) and parameter augmentation/dropout (right) as dual mechanisms for Fisher-Rao volume reduction on statistical manifold $\mathcal{M}$. (\textit{Left}) Data augmentation expands the empirical distribution support in $\eta$-space (perturbing sufficient statistics $\Delta \eta$), which induces an implicit geometric compression in the dual $\theta$-coordinates via the inverse Fisher metric $g^{ij}(\theta)$. (\textit{Right}) Dropout directly constrains the natural parameter trajectory in $\theta$-space through stochastic masking (reducing accessible volume $\sqrt{\det(g_{ij})}$), with stop-gradient ensuring irreversibility. Both operations are operationally equivalent to anisotropic Tikhonov regularization $\|\theta - \theta_0\|_{\mathcal{K}}^2$ in parameter space, unifying input-space and parameter-space regularization.}
    \label{fig:augmentation-correspondence}
\end{figure}

\subsection{Augmentation Correspondence: Data Augmentation and Parameter Augmentation}
\label{subsubsec:augmentation-correspondence}

Fig.~\ref{fig:augmentation-correspondence} illustrates the duality between data-side and parameter-side regularization through the lens of Fisher-Rao geometry. \textbf{Left panel} ($\eta$-space): Data augmentation operates by expanding the empirical distribution support, perturbing sufficient statistics $\Delta\boldsymbol{\eta}$ in the expectation parameter space; via the inverse Fisher metric $g^{ij}(\theta)$, this induces an implicit geometric compression in the dual $\theta$-coordinates. \textbf{Right panel} ($\theta$-space): Dropout and its variants directly constrain the natural parameter trajectory through stochastic masking, explicitly reducing the accessible Fisher-Rao volume $\sqrt{\det(g_{ij})}$. Both mechanisms achieve unified regularization via anisotropic Tikhonov penalization $\|\theta - \theta_0\|_{\mathcal{K}}^2$, revealing that input-space augmentation and parameter-space stochasticity are operationally equivalent manifestations of manifold volume reduction on $\mathcal{M}$.

A fundamental insight bridging data-centric and parameter-centric operations emerges from the reinterpretation of regularization as augmentation. Bouthillier et al.~\cite{bouthillier2015dropout} established the seminal view that \textit{dropout} can be understood as a form of data augmentation: by randomly zeroing out activations (parameter-side operation), the network effectively generates an ensemble of different functions that map the same input $x$ to multiple stochastic outputs. This creates an implicit expansion of the training set in the output space, analogous to how traditional data augmentation generates multiple views $\{\mathcal{A}(x, z)\}_{z \sim \mu}$ of a single sample through input-space transformations $\mathcal{A}: \mathbb{R}^d \times \Omega \to \mathbb{R}^d$.

Recent work on \textit{Deep Augmentation}~\cite{bruel2025deep} refines and extends this correspondence by demonstrating that the efficacy of dropout as augmentation is \textit{layer-dependent}. Rather than applying dropout uniformly, selectively targeting deeper layers (where high-level semantic features reside) yields substantial gains in contrastive learning. This reveals a crucial duality: while input-level augmentation perturbs the raw data distribution in the \textit{input space} (analogous to operations on the $\boldsymbol{\eta}$-coordinates in the data-moment space), layer-targeted dropout perturbs the \textit{representation space} by introducing stochasticity into the parameter trajectory (operations on $\theta$-coordinates).

The information-geometric framework established in Section~\ref{subsec:info-geo} provides the natural language to unify these perspectives. Recall that the Fisher-Rao metric $g_{ij}(\boldsymbol{\theta})$ governs the geometry of the statistical manifold $\mathcal{M}$, where data-space and parameter-space operations constitute dual coordinate systems:

\paragraph{Data Augmentation as $\boldsymbol{\eta}$-Space Perturbation}
Traditional augmentation (e.g., cropping, masking, or PCA-based feature removal) perturbs the empirical data distribution, effectively modifying the expectation parameters ($\boldsymbol{\eta}$-space). In the dual $\theta$-coordinate system, this induces a geometric compression via the inverse Fisher metric $g^{ij}(\boldsymbol{\theta})$, reducing the effective volume of the manifold accessible to the model.

\paragraph{Dropout as $\theta$-Space Perturbation}
Conversely, dropout and its variants (such as Deep Augmentation with stop-gradient) act directly on the natural parameters $\theta$ by masking or quantizing specific components $\theta^i$. Under the Legendre transformation $\eta_i = \partial \psi / \partial \theta^i$, this constrains the trajectory of the model in $\theta$-space, similarly reducing the invariant Fisher-Rao volume $\sqrt{\det(g_{ij})}$.

\paragraph{Unified Augmentation Objective}
Both paradigms achieve regularization through \textit{manifold volume reduction}, yet they operate via complementary mechanisms:
\begin{itemize}
    \item \textbf{Data augmentation} expands the support of the empirical distribution in $\boldsymbol{\eta}$-space, encouraging the model to learn invariances to input transformations;
    \item \textbf{Parameter augmentation} (dropout) restricts the effective capacity of the model in $\theta$-space by preventing co-adaptation of features~\cite{bruel2025deep}, effectively creating an ensemble of sub-models that share parameters but activate different pathways.
\end{itemize}

\paragraph{Irreversibility as the Key Link}
The efficacy of parameter augmentation relies on \textit{stop-gradient}~\cite{bruel2025deep}: by preventing gradient flow through the augmented layer, the perturbation becomes irreversible in $\theta$-space, mirroring how input augmentations irrevocably transform the data in $\boldsymbol{\eta}$-space. This symmetry ensures that both operations genuinely expand the effective training distribution rather than merely regularizing the model.

Thus, data augmentation and parameter augmentation emerge as dual manifestations of the same geometric principle: both perturb the sufficient statistics of the learning problem---one by enriching the data moments ($\Delta \boldsymbol{\eta}$), the other by constraining the natural parameters ($\Delta \theta$)---to reduce the effective complexity of the statistical manifold and improve generalization.

\paragraph{Explicit Regularization Equivalence}
The geometric volume reduction described above admits an equivalent Tikhonov regularization interpretation in $\theta$-space. Consider input perturbations $\tilde{x} = \mathcal{A}(x, \xi)$ with $\mathbb{E}[\xi]=0$ and $\mathrm{Cov}(\xi)=\Sigma$. Expanding the augmented loss $\mathcal{L}_{\mathrm{aug}}$ to second order:
\begin{equation}
\mathcal{L}_{\mathrm{aug}}(\boldsymbol{\theta}) = \mathcal{L}(\boldsymbol{\theta}) + \frac{1}{2}\mathbb{E}_{x}\left[\nabla_x \ell^\top \Sigma \nabla_x \ell\right] + O(\|\xi\|^4),
\end{equation}
where $\nabla_x \ell = J_x^\top \nabla_y \ell$ via the chain rule with $J_x = \frac{\partial \Phi}{\partial x}$ the input Jacobian. Under the Fisher-Rao metric $g_{ij}(\boldsymbol{\theta})$, this induces an implicit parameter-space regularizer:
\begin{equation}
\mathcal{R}(\boldsymbol{\theta}) = \frac{1}{2}\mathrm{tr}\left(\Sigma \cdot J_x^\top H_{\theta} J_x\right) \approx \frac{1}{2}\|\boldsymbol{\theta} - \boldsymbol{\theta}_0\|^2_{\mathcal{K}},
\end{equation}
where $H_{\theta} = \nabla^2_{\theta}\mathcal{L}$ and $\mathcal{K} = J_\theta^\top \mathbb{E}[\nabla_x \ell \nabla_x \ell^\top] J_\theta$ is the induced curvature matrix. This reveals that data augmentation in $\boldsymbol{\eta}$-space (perturbing sufficient statistics) is operationally equivalent to an anisotropic Tikhonov regularization in $\theta$-space constraining the natural parameter trajectory, complementing the geometric view in Fig.~\ref{fig:master-overview}.

\paragraph{Relation to Implicit Semantic Data Augmentation.}
The Tikhonov equivalence derived above establishes that \textit{explicit} input-space augmentation induces an implicit regularizer in $\theta$-space.
A complementary instantiation of this principle is found in \textit{implicit} semantic data augmentation (ISDA)~\cite{wang2021regularizing}.
ISDA operates within the deep feature space by estimating the intra-class covariance matrix $\boldsymbol{\Sigma}$, thereby capturing semantic transformations (e.g., background substitution, viewpoint alteration) without explicit sample generation.
Instead, ISDA minimizes a tight upper bound on the expected cross-entropy loss, yielding a computationally efficient robust objective.
In the information-geometric picture developed here, the covariance $\boldsymbol{\Sigma}$ encodes the local structure of the data manifold in $\boldsymbol{\eta}$-space, and the induced robust loss corresponds to an anisotropic constraint on the parameter trajectory in $\boldsymbol{\theta}$-space.
Consequently, ISDA furnishes an empirical corroboration of the data-parameter correspondence articulated in this section: even when augmentation is performed \textit{implicitly} at the feature level, its regularizing effect is geometrically realized as a volume-reducing constraint on the statistical manifold $\mathcal{M}$.

\subsection{In-Context Learning as Implicit Gradient: The SDFT Mechanism}
\label{subsubsec:sdft-correspondence}

A compelling empirical realization of the data-parameter correspondence emerges from recent advances in self-distillation fine-tuning (SDFT)~\cite{shenfeld2026self}, which reveals that in-context learning (ICL) implicitly induces gradient-like updates in the parameter space without explicit backpropagation through the demonstration data. 

The theoretical foundation for this view stems from the finding that transformers perform in-context learning via implicit gradient descent~\cite{vonoswald2023transformers, dai2023gpt, akyurek2023learning}. Specifically, when a model is conditioned on demonstrations $c$, its forward computation on a query $x$ effectively implements an implicit optimization step on an underlying loss landscape, analogous to parameter updates induced by gradient descent. This establishes that data-space demonstrations $c$ encode \emph{virtual gradient information} that can be transferred to the parameter space.

\begin{figure}[t]
    \centering
    \includegraphics[width=0.48\textwidth]{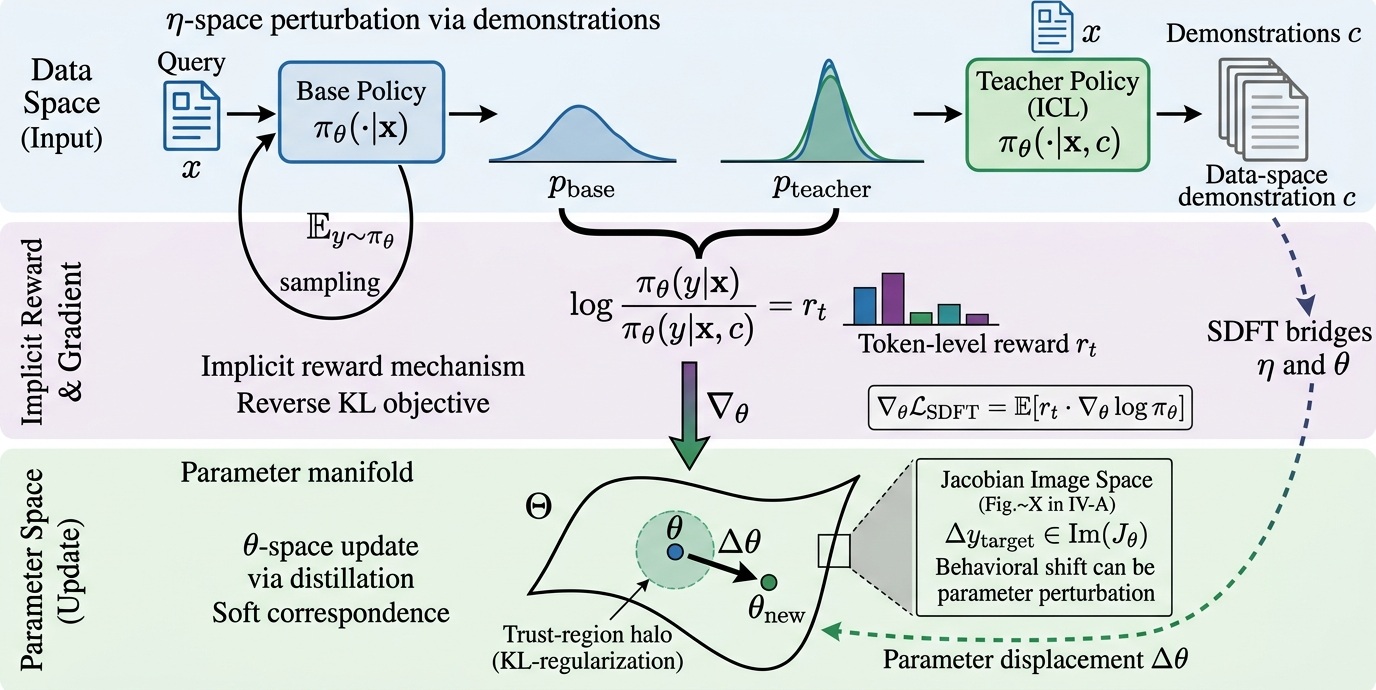}
    \caption{\textbf{SDFT as Data-Parameter Correspondence:} The self-distillation fine-tuning mechanism implements a \textit{soft correspondence} between data-space conditioning (demonstrations $c$) and parameter-space adaptation. (\textit{Top}) Demonstrations $c$ induce a teacher policy $\pi_{\theta}(\cdot|x,c)$ (right) that sharpens the base policy $\pi_{\theta}(\cdot|x)$ (left), creating a distribution shift $\Delta y_{\text{target}}$ in the output space. (\textit{Middle}) The KL divergence ratio $\log\frac{\pi_{\theta}(y|x)}{\pi_{\theta}(y|x,c)}$ serves as an implicit token-level reward $r_t$, enabling on-policy distillation without explicit reward modeling. (\textit{Bottom}) Via the Jacobian image space (Sec.~\ref{subsec:jacobian-framework}), this data-space perturbation induces a parameter displacement $\Delta\theta$ satisfying $J_{\theta}\Delta\theta \approx \Delta y_{\text{target}}$, effectively translating ICL conditioning into gradient-like updates while preserving prior capabilities through KL-regularized trust regions.}
    \label{fig:sdft-mechanism}
\end{figure}

Fig.~\ref{fig:sdft-mechanism} illustrates the three-stage SDFT mechanism that bridges data-space conditioning and parameter-space adaptation. \textbf{Top panel}: Demonstrations $c$ act as a perturbation in $\boldsymbol{\eta}$-space (expectation parameters), inducing a teacher policy $\pi_{\theta}(\cdot|x,c)$ that sharpens the base policy $\pi_{\theta}(\cdot|x)$ and creates a target distribution shift $\Delta y_{\text{target}}$ in the output space. \textbf{Middle panel}: The reverse KL objective extracts an implicit token-level reward $r_t = \log\frac{\pi_{\theta}(y|x)}{\pi_{\theta}(y|x,c)}$ from the data-space conditioning, enabling on-policy distillation without explicit reward modeling. \textbf{Bottom panel}: Through the Jacobian image space framework (Sec.~\ref{subsec:jacobian-framework}), this data-induced behavioral shift is realized as a parameter displacement $\Delta\theta$ satisfying $J_\theta \Delta\theta \approx \Delta y_{\text{target}}$, effectively translating ICL conditioning into gradient-like updates within a KL-regularized trust region that preserves prior capabilities.

The SDFT framework exploits this duality between the model's behavior under different input conditions: consider the base student policy $\pi_{\theta}(\cdot|x)$ operating on raw inputs $x \in \mathcal{X}$, and the teacher policy $\pi_{\theta}(\cdot|x,c)$ conditioned on expert demonstrations $c \in \mathcal{D}$. Mathematically, the demonstration $c$ acts as a perturbation in the data space (analogous to the $\boldsymbol{\eta}$-coordinates in the information-geometric framework of Section~\ref{subsubsec:augmentation-correspondence}), shifting the model's output distribution toward task-optimal behavior. The key insight is that this data-space perturbation induces a corresponding displacement in the parameter manifold $\Theta$ through the reverse KL divergence objective:
\begin{equation}
\begin{aligned}
\nabla_{\theta}\mathcal{L}_{\text{SDFT}} = \mathbb{E}_{y\sim\pi_{\theta}}\Bigg[ &\sum_{t}\log\frac{\pi_{\theta}(y_{t}|y_{<t},x)}{\pi_{\theta}(y_{t}|y_{<t},x,c)} \\
&\times \nabla_{\theta}\log\pi_{\theta}(y_{t}|y_{<t},x)\Bigg].
\end{aligned}
\end{equation}
This objective effectively translates the data-side operation (demonstration conditioning) into a parameter-side update. From the Jacobian Image Space perspective introduced in Section~\ref{subsec:jacobian-framework}, the ICL-conditioned teacher $\pi(\cdot|x,c)$ resides in the image space $\operatorname{Im}(J_\theta)$, meaning the behavioral shift induced by the demonstration can be realized as a parameter perturbation $\Delta\theta$ such that $J_\theta \Delta\theta \approx \Delta y_{\text{target}}$, where $\Delta y_{\text{target}}$ represents the desired output modification.

\begin{table*}[t]
\caption{Roadmap of Novel Data-Parameter Correspondences in Section~IV}
\label{tab:roadmap_iv}
\centering
\footnotesize
\setlength{\tabcolsep}{5pt}
\renewcommand{\arraystretch}{1.15}
\begin{tabular}{@{}clllp{6.2cm}@{}}
\toprule
\textbf{Sub.} & \textbf{Correspondence} & \textbf{Stage} & \textbf{Dim.} & \textbf{Core Insight} \\
\midrule
\ref{subsec:jacobian-framework} & Frozen-Weight Inference & 3 (Inference) & Structural & Jacobian image spaces $\operatorname{Im}(J_x)$ and $\operatorname{Im}(J_\theta)$ establish local linear duality for immutable $\theta_0$ \\
\ref{subsec:low-rank} & Low-Rank Subspace & 3 $\to$ 1 & Structural & ICL (sample-induced) and LoRA (parameterized) explore identical Grassmannian $\mathcal{G}(r,d)$; $k$-shot $\approx$ rank-$r$ budget equivalence \\
\midrule
\ref{subsec:continual} & Continual Learning & 1 (Training) & Temporal & Replay ($\eta$-space support) and EWC ($\theta$-space curvature) impose dual trajectory constraints on $\mathcal{M}$; buffer size $k \propto 1/\lambda$ \\
%\ref{subsec:token-shapley} & Cooperative Game & 2 (Post-Training) & Attribution & Model Shapley (parameter importance) and Token Shapley (vocabulary importance) share symmetric second-order Fisher-aware scores under frozen $\theta^*$ \\
\midrule
\ref{subsec:conceptual-security} & Cooperative Attack & 1 $\to$ 2 & Security (Offensive) & Data poisoning $\Delta\mathcal{D}$ and parameter backdoors $\Delta\theta$ exhibit synergistic min-min amplification; sequential imprinting $\to$ solidification \\
\ref{subsec:privacy-correspondence} & Privacy Cascading & 1 $\to$ 2 & Privacy & Data compression $\rho$ and parameter protection $\epsilon_P$ follow sub-additive composition $\epsilon_{\text{eff}} \leq \epsilon_D + \rho\epsilon_P$ (cascading vs. cooperative) \\
\ref{subsec:testing-defense} & Testing-Time Defense & 3 (Inference) & Security (Defensive) & Product constraint $\epsilon_d \cdot \gamma_p \leq C_{\text{budget}}$ couples input sanitization and Lipschitz regularization; $\gamma_p < 1$ necessary for long-form safety \\
\midrule
\ref{subsec:composition} & Composition/Merging & 2 (Post-Training) & Compositional & Data mixing ($\eta$-space Bregman barycenter) and model merging ($\theta$-space barycenter) are dual geodesic interpolations; \textbf{independent from security} \\
\midrule
\ref{subsec:synthesis} & Manifold Synthesis & All stages & Unified & Fisher--Rao metric $g_{ij}(\theta)$ unifies all correspondences: data and parameters as dual coordinates ($\theta$ vs. $\eta$) on $\mathcal{M}$ \\
\bottomrule
\end{tabular}
\\[3pt]
\scriptsize \textit{Note:} Stage~3=Inference (frozen $\theta$), Stage~1=Training (plastic $\theta$), Stage~2=Post-Training (frozen $\theta$, calibration only). Dim.=Dimensional classification.
\end{table*}

Thus, SDFT exemplifies a practical instantiation of the data-parameter correspondence: data-side operations (demonstration-based ICL) and parameter-side operations (fine-tuning via distillation) converge through the shared geometric structure of the statistical manifold. This bypasses the need for explicit reward inference while preserving prior capabilities through the trust-region property inherent in the KL-regularized objective. Moreover, it establishes a concrete operational symmetry with data augmentation (Section~\ref{subsubsec:augmentation-correspondence}): while augmentation perturbs the empirical distribution in $\boldsymbol{\eta}$-space to regularize the loss landscape, SDFT demonstrates that conditioning on demonstrations in $\boldsymbol{\eta}$-space induces a specific parameter trajectory in $\theta$-space, effectively treating the data-context pair $(x,c)$ as a proxy for the optimal update direction implied by implicit gradient descent~\cite{vonoswald2023transformers}.

\section{Novel Data-Parameter Correspondences}
\label{sec:novel}

\noindent\textbf{Roadmap\footnote{\noindent\textit{Reading Guide:} The novel correspondences in this section extend the dual-coordinate framework ($\eta$-space vs $\theta$-space, Section~II-B) to low-rank adaptation, continual learning, and security contexts.}.} 
This section establishes eight novel data-parameter correspondences organized across five functional dimensions: \emph{structural inference} (IV-A/B), \emph{adaptive learning} spanning Stage~1 (training) and Stage~2 (post-training) (IV-C), \emph{security-robustness} (IV-D/E/F), \emph{compositional fusion} (IV-H), and \emph{manifold synthesis} (IV-G). 
Table~\ref{tab:roadmap_iv} provides a systematic navigational guide. 

The progression from IV-A to IV-H reflects a conceptual arc from \emph{local operational constraints} (Jacobian image spaces, low-rank subspaces) through \emph{temporal adaptation mechanisms} (continual learning, attribution) and \emph{adversarial robustness} (attack cooperation, privacy cascading, defense constraints) to \emph{compositional geometry} (barycentric interpolation) and final \emph{information-geometric unification}. 
Specifically, IV-C addresses Stage~1 (training-time plasticity), forming a temporal continuum within the learning dimension. 
IV-H (composition) operates at the same hierarchical level as the security triad (IV-D/E/F), addressing cross-sectional fusion rather than adversarial robustness.

\subsection{Frozen-Weight Inference via Jacobian Image Spaces}
\label{subsec:jacobian-framework}
The Fisher-Rao geometric framework (Section~\ref{subsec:geometric-correspondence}) characterizes the global structure of the statistical manifold $\mathcal{M}$ under parameter plasticity, applicable to training (Stage 1) and post-training calibration (Stage 2). However, Stage 3 (Inference) operates under frozen-parameter constraints: the weight vector $\boldsymbol{\theta}_0 \in \Theta$ remains immutable, gradient computation is infeasible, and any adaptation must be ephemeral---instantaneous and non-persistent. 

We introduce the \textbf{Jacobian Image Space Framework} as the natural computational realization of data-parameter duality for this regime. This framework provides the local linear approximation to the global Fisher-Rao geometry, unifying data-space interventions (e.g., in-context learning~\cite{shenfeld2026self}, prompt engineering) and parameter-space interventions (e.g., dynamic LoRA routing, adapter gating) through their shared action on the output tangent space.

\paragraph{First-Order Variational Decomposition.}
Consider the model $\Phi: \mathcal{X} \times \Theta \to \mathcal{Y}$ evaluated at fixed $(x, \boldsymbol{\theta}_0)$. The first-order variation decomposes output perturbations into three orthogonal contributions:
\begin{equation}
\Delta y = \underbrace{J_x(x, \boldsymbol{\theta}_0)\Delta x}_{\text{Data-side}} + \underbrace{J_\theta(x, \boldsymbol{\theta}_0)\Delta\boldsymbol{\theta}}_{\text{Parameter-side}} + \underbrace{\Delta y_{\text{direct}}}_{\text{Decoder-side}},
\label{eq:jacobian-decomp}
\end{equation}
where $J_x = \frac{\partial\Phi}{\partial x}$ and $J_\theta = \frac{\partial\Phi}{\partial\boldsymbol{\theta}}$ are the input and parameter Jacobians, respectively; $\Delta y_{\text{direct}}$ subsumes both higher-order residual terms and direct output-space interventions (e.g., logit bias, classifier-free guidance) that bypass the input and parameter tangent spaces. This decomposition reveals that data-space and parameter-space operations constitute dual coordinates on the shared output tangent space $\mathcal{T}_y\mathcal{M}$.

\begin{figure}[h]
    \centering
    \includegraphics[width=0.49\textwidth]{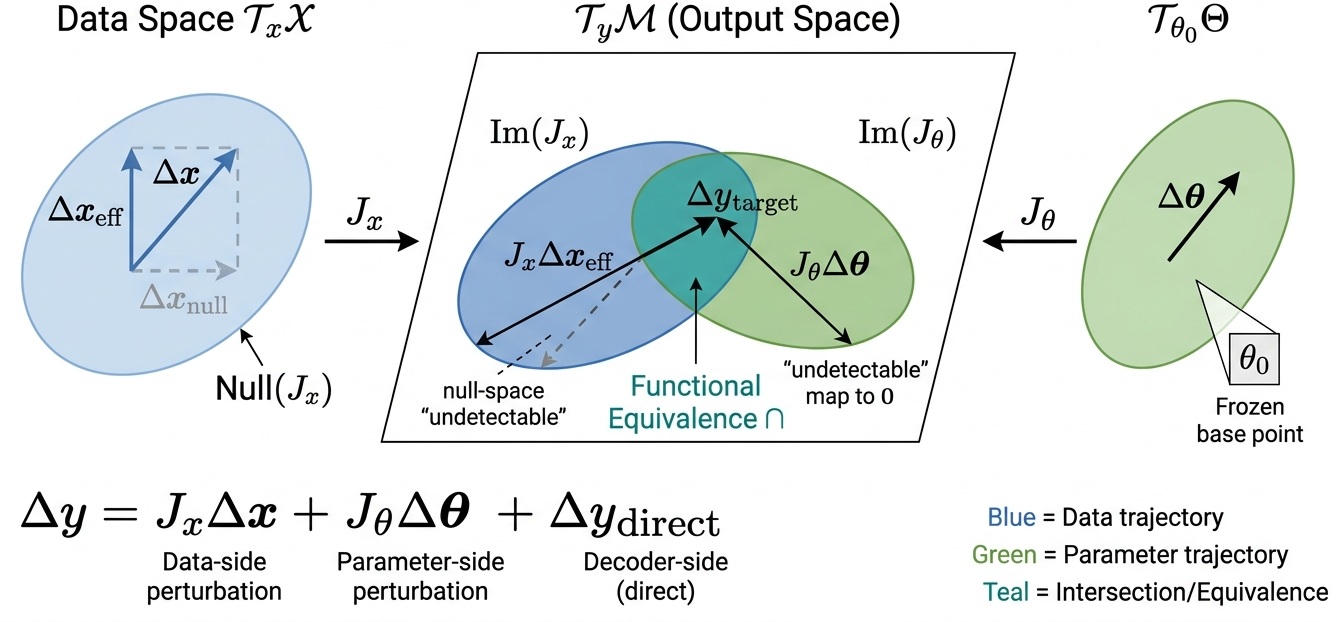}\vspace{-0.1in}
    \caption{\textbf{Jacobian Image Space Framework:} Local linear duality between data-space and parameter-space operations under frozen weights (Stage 3). The input Jacobian $J_x$ and parameter Jacobian $J_\theta$ map perturbations from their respective tangent spaces onto the shared output tangent space $T_y\mathcal{M}$. The intersection $\operatorname{Im}(J_x) \cap \operatorname{Im}(J_\theta)$ (teal region) represents the \textit{functional equivalence} set where targets can be achieved via either prompting ($\Delta x$) or transient adaptation ($\Delta\theta$). Null spaces $\operatorname{Null}(J_x)$ and $\operatorname{Null}(J_\theta)$ (dashed arrows) induce geometric undetectability---perturbations along these directions produce no output change, establishing fundamental limits for adversarial detection. When $\Delta y_{\text{target}}$ lies outside the joint image space, direct decoder-side manipulation ($\Delta y_{\text{direct}}$) becomes necessary.}
    \label{fig:jacobian-image-space}
\end{figure}

Fig.~\ref{fig:jacobian-image-space} visualizes the local linear duality established in Eq.~\eqref{eq:jacobian-decomp}. \textbf{Left (Data Space $\mathcal{T}_x\mathcal{X}$)}: Input perturbations $\Delta x$ decompose into effective components $\Delta x_{\text{eff}} \in \operatorname{Row}(J_x)$ (solid arrow) and null-space components $\Delta x_{\text{null}} \in \operatorname{Null}(J_x)$ (dashed arrow); only the former induces output change via $J_x\Delta x_{\text{eff}}$. \textbf{Right (Parameter Space $\mathcal{T}_{\theta_0}\Theta$)}: Parameter perturbations $\Delta\theta$ similarly map through $J_\theta$, with $\operatorname{Null}(J_\theta)$ representing frozen-weight directions that leave outputs invariant. \textbf{Center (Output Space)}: The intersection $\operatorname{Im}(J_x) \cap \operatorname{Im}(J_\theta)$ (teal region) defines \emph{functional equivalence}---targets $\Delta y_{\text{target}}$ in this set are achievable via either prompting ($\Delta x$) or transient adaptation ($\Delta\theta$), quantifying the operational symmetry between data-side and parameter-side interventions under frozen weights.

\paragraph{Functional Equivalence via Image Space Intersection}
The feasible set of inference-time interventions is governed by the \emph{image spaces} $\operatorname{Im}(J_x)$ and $\operatorname{Im}(J_\theta)$. For any target output variation $\Delta y_{\text{target}}$ lying in the intersection $\operatorname{Im}(J_x) \cap \operatorname{Im}(J_\theta)$, there exist equivalent realizations via either data perturbation $\Delta x$ (e.g., in-context learning~\cite{brown2020language,min2022rethinking}, prompt engineering~\cite{liu2023pre}) or instantaneous parameter perturbation $\Delta\theta$ (e.g., dynamic LoRA routing~\cite{hu2021lora}, adapter gating~\cite{houlsby2019parameter}, or multi-adapter mixing~\cite{pfeiffer2021adapterfusion}):
\begin{equation}
\begin{aligned}
&\exists\, \Delta x \in \mathcal{T}_x\mathcal{X},\; \Delta\theta \in \mathcal{T}_{\theta_0}\Theta \\
&\quad\text{such that}\quad 
J_x\Delta x = J_\theta\Delta\theta = \Delta y_{\text{target}}.
\end{aligned}
\end{equation}
This functional equivalence implies that the choice between demonstration-based prompting (Table~\ref{tab:stages}, Stage~3 Data-Side) and transient parameter modulation (Stage~3 Parameter-Side) is determined by geometric alignment: the optimal path corresponds to which image space provides the minimal-norm solution for the specific $\Delta y_{\text{target}}$.

\paragraph{Null Spaces and Geometric Undetectability}
The null spaces $\operatorname{Null}(J_x)$ and $\operatorname{Null}(J_\theta)$ characterize perturbations that are \emph{orthogonal to all effective output directions}, producing zero change in $\Delta y$ regardless of magnitude. This yields a geometric notion of \textbf{undetectability}: an adversarial perturbation $\Delta x = \Delta x_{\text{eff}} + \Delta x_{\text{null}}$, with $\Delta x_{\text{eff}} \in \operatorname{Row}(J_x)$ and $\Delta x_{\text{null}} \in \operatorname{Null}(J_x)$, achieves the target objective through the effective component while the null-space component provides masking. Detection schemes monitoring only output changes cannot distinguish $\Delta x$ from $\Delta x + \Delta x_{\text{null}}$, establishing a fundamental limit on adversarial detectability in frozen-weight regimes (see Section~\ref{subsec:adversarial} for security implications).

\paragraph{Intermediate Layer Operations.}
The framework extends to intermediate-layer interventions. Feature calibration (e.g., dynamic normalization) corresponds to perturbations in the input space of layer $l$, propagated via the chain rule through subsequent Jacobians $J_{x}^{(l)} = \frac{\partial h_L}{\partial h_l}$~\cite{bruel2025deep}. Similarly, KV cache manipulation---reordering position encodings or adjusting causal masks---constitutes structured perturbations in the input tangent space of attention layers~\cite{vaswani2017attention,dai2019transformer}.

\paragraph{Decoder-Side Correction as Complementary Space.}
When the target variation lies outside the joint image space, $\Delta y_{\text{target}} \notin \operatorname{Im}(J_x) + \operatorname{Im}(J_\theta)$, neither data-side nor parameter-side operations suffice. This necessitates \emph{direct output manipulation} (Table~\ref{tab:stages}, Stage~3 decoding operations such as classifier-free guidance~\cite{ho2022classifier}, self-debiasing~\cite{schick2021self}, or logit biasing), corresponding to the $\Delta y_{\text{direct}}$ term in Eq.~\eqref{eq:jacobian-decomp}. Energy-guided sampling methods (e.g., Langevin dynamics in latent space)~\cite{du2020energy,xiao2021energy} similarly bypass forward Jacobian constraints via iterative optimization in the output space. A concrete example is FlashSampling~\cite{ruiz2026flasampling}, which fuses sampling into the LM-head to accelerate decoding without altering $x$ or $\theta$, thereby operating purely in $\Delta y_{\text{direct}}$.

\paragraph{Computational Realization via Automatic Differentiation}
Explicit Jacobian materialization is prohibitive for large models. The framework operationalizes through \textbf{Jacobian-Vector Products} (JVP) and \textbf{Vector-Jacobian Products} (VJP)~\cite{frostig2018compiling}:
\begin{equation}
J_x \cdot v = \left.\frac{d}{d\epsilon}\Phi(x+\epsilon v, \theta_0)\right|_{\epsilon=0}, 
\qquad 
u^\top \cdot J_\theta = \text{VJP}_{\theta_0}(u),
\end{equation}
enabling efficient projection onto $\operatorname{Im}(J_x)$ and $\operatorname{Null}(J_x)$ without forming full matrices. This ensures the framework is algorithmically viable for inference-time alignment.

\paragraph{Relation to Global Information Geometry}
The Jacobian framework serves as the \emph{local linearization} of the Fisher-Rao structure. While the Fisher metric $g_{ij}(\theta) = \mathbb{E}_{x\sim p(x)}[(J_\theta^\top J_\theta)_{ij}]$ captures the expected geometry over the data distribution, the instantaneous image spaces $\operatorname{Im}(J_x(x,\theta_0))$ and $\operatorname{Im}(J_\theta(x,\theta_0))$ describe the \emph{operational constraints} at specific points $(x, \theta_0)$. Together, they form a complete hierarchy: Fisher-Rao provides the population-level duality between $\theta$ and $\boldsymbol{\eta}$, while the Jacobian framework provides the sample-level computational mechanism for Stage~3 alignment under immutable parameters.

\begin{figure}[t]
    \centering
    \includegraphics[width=0.95\columnwidth]{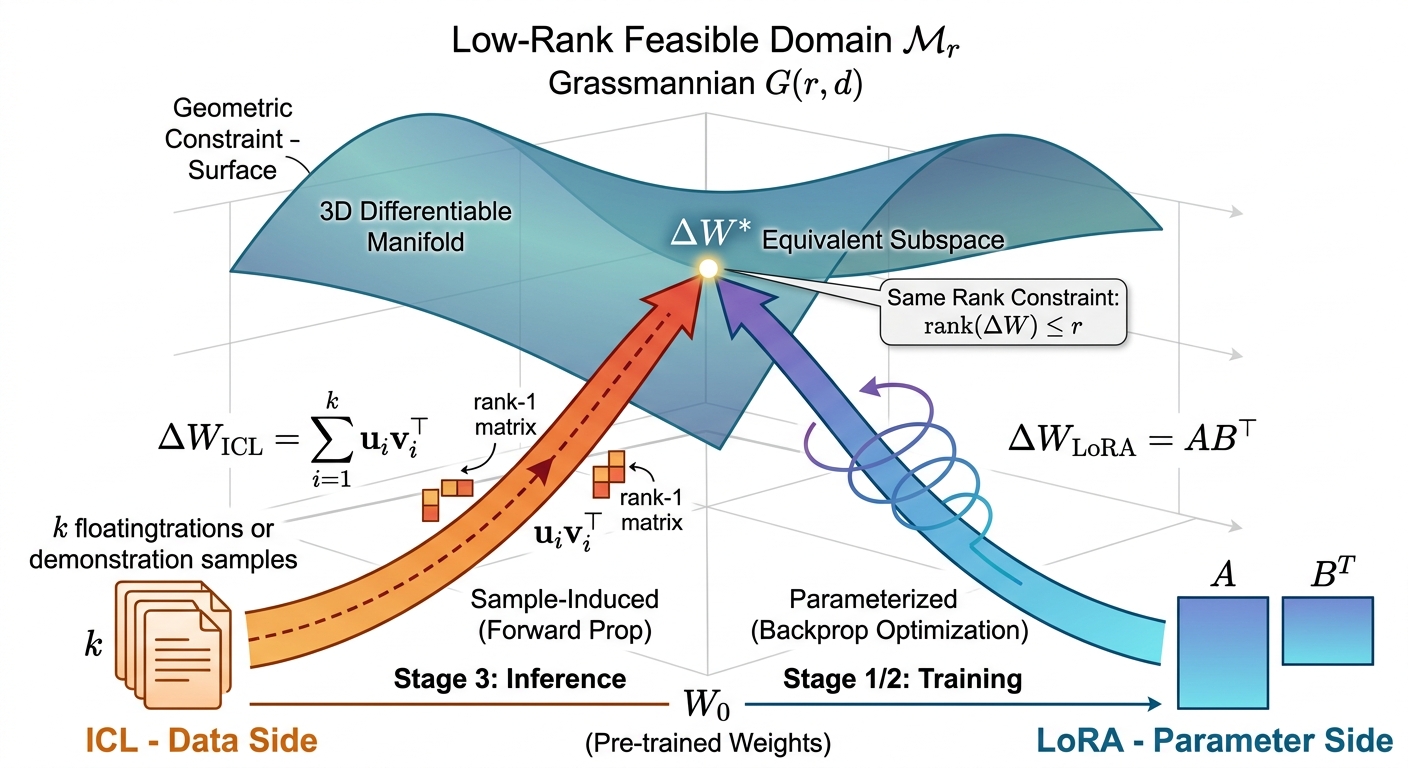}\vspace{-0.1in}
    \caption{The Low-Rank Correspondence: ICL (left) induces low-rank updates via sample composition (forward induction), while LoRA (right) explicitly parameterizes the same subspace via gradient optimization. Both operate on the identical geometric constraint $\mathcal{M}_r$.}
    \label{fig:iclora}
\end{figure}

\subsection{Low-Rank Correspondence: Sample-Induced vs. Parameterized Subspace Adaptation}
\label{subsec:low-rank}

While Section~\ref{subsec:jacobian-framework} establishes the local linear equivalence between data-space and parameter-space interventions via Jacobian image spaces, a more specific yet pervasive structural correspondence emerges when we restrict our attention to \emph{low-rank adaptation}---arguably the most prevalent form of parameter-efficient fine-tuning in modern LLMs. We observe that \emph{in-context learning} (ICL), a data-side operation that conditions the model on demonstration examples without parameter updates~\cite{brown2020language,min2022rethinking}, and \emph{low-rank adaptation} (LoRA), a parameter-side operation that explicitly factorizes weight updates~\cite{hu2021lora,zhang2023adalora}, constitute dual parameterizations of the same geometric object: a rank-constrained perturbation $\Delta W$ to the pre-trained weight matrix $W_0$.

Fig.~\ref{fig:iclora} illustrates this structural equivalence on the low-rank feasible domain $\mathcal{M}_r \subseteq \mathcal{G}(r,d)$. \textbf{Left (ICL/Data Side)}: Operating in Stage~3 with frozen weights $W_0$, ICL induces $\Delta W_{\text{ICL}} = \sum_{i=1}^k \boldsymbol{u}_i \boldsymbol{v}_i^\top$ via forward composition over $k$ demonstration samples~\cite{vonoswald2023transformers, akyurek2023learning}, effectively locating a subspace through sampling. \textbf{Right (LoRA/Parameter Side)}: Operating in Stage~1/2, LoRA explicitly parameterizes $\Delta W_{\text{LoRA}} = \boldsymbol{A}\boldsymbol{B}^\top$ and refines the subspace via backpropagation. \textbf{Geometric Unity}: Both operate on the identical Grassmannian constraint $\mathcal{G}(r,d)$, establishing that $k$-shot demonstration and rank-$r$ adaptation explore the same geometric object---one through sample-induced probing, the other through gradient-based optimization.

\subsubsection*{Dual Perspectives on Rank-Constrained Adaptation}

This constrained optimization admits two distinct operational realizations across the data-parameter boundary, corresponding respectively to explicit parameter-side optimization (LoRA) and implicit data-side approximation (ICL). Formally, both can be understood as solving a unified low-rank objective on the statistical manifold $\mathcal{M}$:
\begin{equation}
\min_{\Delta W} \mathcal{L}(W_0 + \Delta W) \quad \text{s.t.} \quad \operatorname{rank}(\Delta W) \leq r, \quad r \ll d,
\end{equation}
where $\mathcal{L}$ denotes the task loss and $d$ the parameter dimension. The two dual mechanisms are summarized in Table~\ref{tab:low-rank}.

\begin{table*}[t]
\centering
\footnotesize
\setlength{\tabcolsep}{3pt}
\caption{Low-Rank Correspondence: ICL as Sample-Induced vs. LoRA as Parameterized Subspace Adaptation}
\label{tab:low-rank}
\begin{tabular}{@{}p{2.2cm}p{3.5cm}p{3.5cm}p{4.5cm}@{}}
\toprule
\textbf{Dimension} & \textbf{Data-Side (ICL)} & \textbf{Parameter-Side (LoRA)} & \textbf{Unified Geometric Interpretation} \\
\midrule
Operational Space & Demonstration set $\{x_i\}_{i=1}^k$ (Stage~3) & Low-rank matrices $\boldsymbol{A} \in \mathbb{R}^{d \times r}, \boldsymbol{B} \in \mathbb{R}^{r \times d}$ (Stage~1/2) & Grassmannian manifold $\mathcal{G}(r,d)$ of $r$-dimensional subspaces \\
\midrule
Parameterization & $\Delta W_{\text{ICL}} = \sum_{i=1}^k \boldsymbol{u}_i \boldsymbol{v}_i^\top$ (implicit, forward-induced) & $\Delta W_{\text{LoRA}} = \boldsymbol{A}\boldsymbol{B}^\top$ (explicit, backprop-optimized) & Rank constraint $\operatorname{rank}(\Delta W) \leq r$ \\
\midrule
Optimization Mechanism & \textbf{Sampling}: Select $k$ informative examples to \emph{locate} a subspace & \textbf{Refinement}: Gradient descent to \emph{optimize} within the subspace & Sampling-optimization duality on $\mathcal{M}_r = \{W : \operatorname{rank}(W-W_0) \leq r\}$ \\
\midrule
Budget Analogy & Context length $k$ (sample budget) & Rank $r$ (parameter budget) & Effective degrees of freedom $k \approx r$ \\
\bottomrule
\end{tabular}
\end{table*}

\paragraph*{Structural Insight}
The ICL paradigm induces a low-rank update through the composition of attention mechanisms over demonstration tokens. Recent theoretical analysis reveals that transformers perform ICL via implicit gradient descent~\cite{vonoswald2023transformers,dai2023gpt,akyurek2023learning}, where the forward computation on demonstration examples $c$ effectively implements an optimization step on an underlying loss landscape. This induces a rank-structured perturbation $\Delta W_{\text{ICL}}$ spanned by the outer products of key-query interactions from the demonstration set.

Conversely, LoRA explicitly parameterizes the update as $\Delta W = \boldsymbol{A}\boldsymbol{B}^\top$, where the trainable matrices $\boldsymbol{A}$ and $\boldsymbol{B}$ are optimized via backpropagation~\cite{hu2021lora}. While AdaLoRA~\cite{zhang2023adalora} subsequently introduces adaptive budget allocation for the rank $r$, the fundamental geometric operation remains the exploration of the same low-rank submanifold $\mathcal{M}_r$.

\subsubsection*{Connection to Jacobian Image Spaces}

This low-rank correspondence specializes the general Jacobian duality established in Section~\ref{subsec:jacobian-framework}. While Eq.~\eqref{eq:jacobian-decomp} demonstrates that arbitrary perturbations satisfy $J_x \Delta x \approx J_\theta \Delta \theta$ in the output tangent space, the ICL-LoRA correspondence reveals that when we restrict $\Delta \theta$ to lie in a low-rank subspace (the column space of $\boldsymbol{A}$ in LoRA), the equivalent data-side perturbation $\Delta x$ (the demonstrations in ICL) effectively \emph{probe} this subspace through the Jacobian image $\operatorname{Im}(J_x) \cap \operatorname{Im}(J_\theta)$.

Specifically, the $k$ demonstration examples in ICL span an approximate subspace in the input tangent space $\mathcal{T}_x\mathcal{X}$, which, when propagated through $J_x$, projects onto the same low-rank submanifold of $\mathcal{T}_\theta\Theta$ that LoRA explicitly parameterizes. This suggests that ICL serves as a \emph{sample-driven subspace probe}, whereas LoRA serves as a \emph{parameterized subspace optimizer}---both operating within the identical geometric constraint of rank $\leq r$.

\subsubsection*{Gradient Interaction in Low-Rank Regimes}

This correspondence further refines the gradient interaction framework of Section~\ref{subsec:gradient-interaction}. The Gradient Interaction Matrix $\boldsymbol{M} \in \mathbb{R}^{N \times K}$ (Eq.~\eqref{eq:gradient_interaction_matrix}) characterizes the alignment between training samples and parameter dimensions. Under the low-rank constraint:

\begin{itemize}
    \item \textbf{Data Utility (Row-wise)}: The importance of a demonstration sample $n$ for ICL corresponds to its alignment with the dominant singular vectors of $\boldsymbol{M}$, effectively identifying which samples induce the most significant rank-1 components $\boldsymbol{u}_n \boldsymbol{v}_n^\top$.
    \item \textbf{Parameter Importance (Column-wise)}: In LoRA, the column-wise aggregation identifies which parameter dimensions (singular directions) should be retained in the low-rank factorization to maximize validation utility.
\end{itemize}

Thus, the singular value decomposition of the gradient interaction structure $\boldsymbol{M}$ provides the bridge: the left singular vectors guide ICL sample selection, while the right singular vectors guide LoRA subspace initialization.

\subsubsection*{Practical Implications: Hybrid Adaptation}

Recognizing this correspondence enables synergistic strategies unattainable through isolated consideration:

\paragraph*{Coarse-to-Fine Optimization} One may employ ICL to rapidly \emph{locate} a promising low-rank subspace (coarse selection via $k$ samples), followed by LoRA to \emph{refine} the solution within that subspace (fine optimization via gradient descent). This hybrid approach leverages the computational efficiency of frozen-weight inference (Stage~3) for subspace identification before committing to parameter updates (Stage~1).

\paragraph*{Budget Correspondence}
The framework suggests a qualitative correspondence between sample budget and rank budget: a $k$-shot ICL context and a rank-$r$ LoRA both operate within an effective subspace whose dimensionality is bounded by $\min(k, r)$. This observation offers a principled heuristic for cross-paradigm resource allocation---when inference latency dominates, increasing LoRA rank $r$ may substitute for additional ICL shots; when parameter storage dominates, increasing ICL shots $k$ may reduce the required LoRA rank.

\paragraph*{Subspace Alignment Verification}
Empirically, this correspondence can be tested by measuring the principal angles between the subspace spanned by ICL-induced updates (via SVD of the implicit gradient sum) and the column space of LoRA matrices $\boldsymbol{A}$. Convergence of these subspaces as $T \to \infty$ in LoRA training would indicate that both modalities ultimately explore the same dominant directions of the Fisher-Rao metric on $\mathcal{M}$.

\begin{remark}
This low-rank correspondence is distinct from the implicit gradient descent interpretation of ICL~\cite{vonoswald2023transformers}. While the latter focuses on the \emph{optimization dynamics} (how ICL mimics gradient steps), our focus here is on the \emph{geometric structure} (both methods inhabit the same rank-constrained feasible set). The optimization trajectory (algebraic in ICL, dynamic in LoRA) differs, but the terminal reachable submanifold is identical.
\end{remark}

\subsection{Continual Learning Correspondence: Data Replay as Implicit Parameter Regularization}
\label{subsec:continual}

\begin{figure}[t]
    \centering
    \includegraphics[width=0.95\columnwidth]{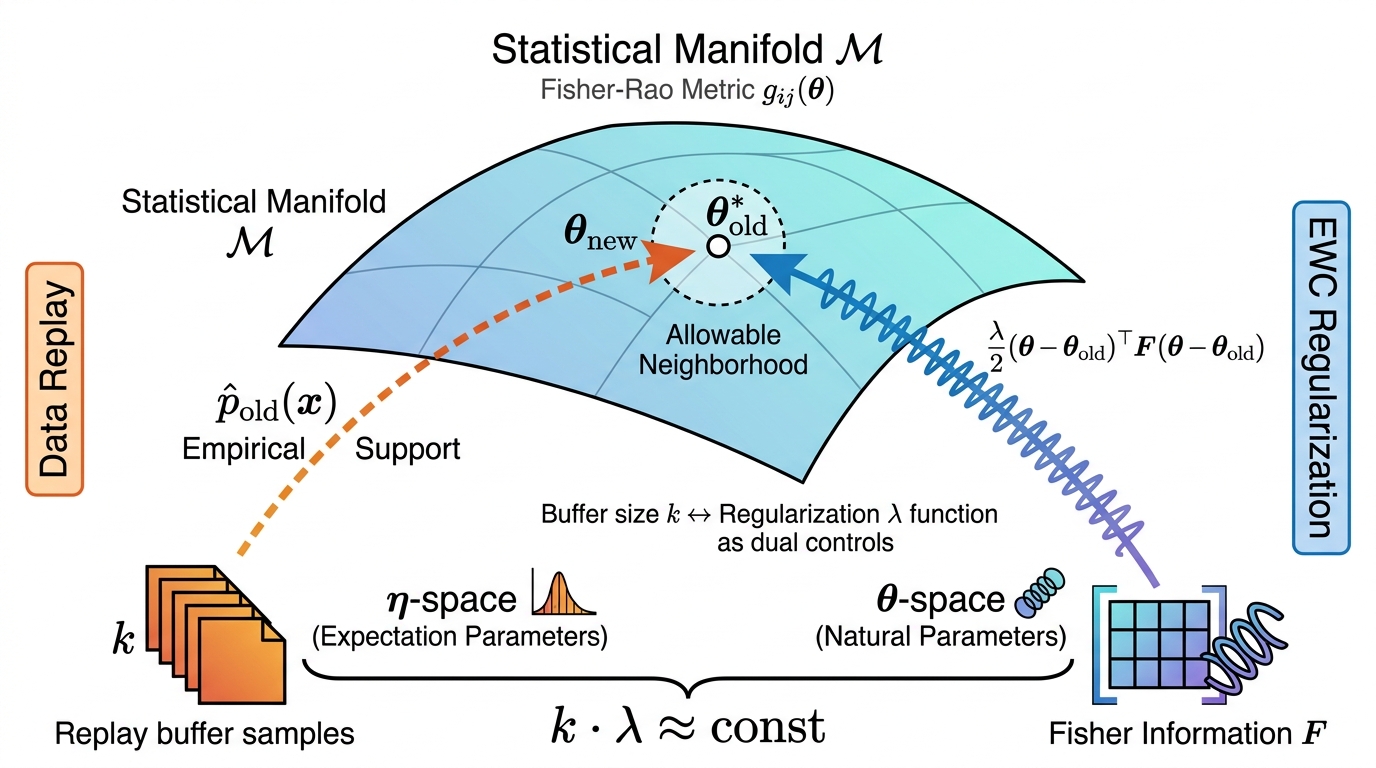}\vspace{-0.1in}
    \caption{The Continual Learning Correspondence: Data replay (left) constrains the learning trajectory by maintaining empirical distribution support in $\boldsymbol{\eta}$-space, while EWC (right) imposes curvature constraints via Fisher Information in $\theta$-space. Both mechanisms define the same allowable neighborhood around $\theta_{\text{old}}^*$ on manifold $\mathcal{M}$, with buffer size $k$ and regularization strength $\lambda$ acting as dual budget variables.}
    \label{fig:continue}
\end{figure}

\textbf{Boundary with existing content:} 
This section extends the data-parameter correspondence to the \emph{sequential learning} setting, distinct from prior sections in two aspects: (1) Unlike Section~\ref{subsubsec:augmentation-correspondence} which examines single-task adaptation via implicit gradient descent, we address \emph{multi-task} knowledge retention; (2) Unlike Section~\ref{subsec:low-rank} which focuses on \emph{structural} constraints (subspace dimensionality) for static adaptation, we examine \emph{trajectory} constraints (preventing deviation from previous task optima).

Fig.~\ref{fig:continue} visualizes the temporal duality in continual learning on the statistical manifold $\mathcal{M}$. \textbf{Left ($\boldsymbol{\eta}$-space)}: Data replay maintains an episodic buffer of $k$ past samples, preserving empirical distribution support in the expectation parameter space; this constrains the learning trajectory by ensuring the model remains within the moment space of previously observed data. \textbf{Right ($\theta$-space)}: Elastic Weight Consolidation (EWC) imposes curvature constraints via the Fisher Information Matrix $\boldsymbol{F}$, penalizing parameter movements along high-sensitivity directions in the natural parameter space. \textbf{Dual Budget}: Buffer size $k$ and regularization strength $\lambda$ act as interchangeable controls---small $k$ necessitates large $\lambda$ and vice versa---both defining the same allowable neighborhood around the old task optimum $\theta_{\text{old}}^*$ on $\mathcal{M}$ under the Fisher-Rao metric $g_{ij}(\theta)$.

\subsubsection*{Problem Setup: Two Camps in Continual Learning}

In continual learning (also known as lifelong learning), a model must sequentially learn tasks $\mathcal{T}_1, \mathcal{T}_2, \dots, \mathcal{T}_T$ while avoiding catastrophic forgetting of previous knowledge. Existing methodologies divide into two operational camps:

\paragraph*{Data-Side: Experience Replay} 
This approach maintains a fixed-size episodic memory buffer $\mathcal{B}$ (alternatively denoted as replay buffer with size $k$) storing exemplars from previous tasks. When learning new task $\mathcal{T}_t$, the model trains on both $\mathcal{T}_t$'s data and the replay buffer $\mathcal{B}$, effectively maintaining empirical support of past data distributions~\cite{shenfeld2026self}.

\paragraph*{Parameter-Side: Parameter Regularization} 
Represented by Elastic Weight Consolidation (EWC)~\cite{kirkpatrick2017overcoming} and its variants, this approach imposes quadratic constraints on parameter updates based on Fisher Information:
\begin{equation}
\mathcal{L}_{\text{CL}}(\theta) = \mathcal{L}_{\text{new}}(\theta) + \frac{\lambda}{2} \sum_k F_{kk}(\theta_{\text{old}})(\theta_k - \theta_{\text{old},k})^2,
\label{eq:ewc_loss}
\end{equation}
where $F_{kk} = \mathbb{E}_{x \sim \mathcal{D}_{\text{old}}} [(\partial \log p_\theta(x)/\partial \theta_k)^2]$ measures the sensitivity of parameter $k$ to the old data distribution, and $\lambda$ is the plasticity-stability trade-off coefficient per Table~\ref{tab:notation}.

Both strategies aim to constrain the learning trajectory on statistical manifold $\mathcal{M}$ to remain within an ``allowable neighborhood'' of the old task optimum $\theta_{\text{old}}^*$, yet they constitute dual implementations: \emph{replay} maintains empirical distribution support in $\boldsymbol{\eta}$-space (expectation parameters), while \emph{EWC} imposes curvature penalties directly in $\theta$-space (natural parameters).

\subsubsection*{Geometric Duality: Support Maintenance vs. Curvature Constraint}

Within the information geometric framework established in Section~\ref{subsec:info-geo} and Section~\ref{subsec:low-rank}:

\paragraph*{Data Replay as $\boldsymbol{\eta}$-Space Support}
The replay buffer $\mathcal{B}$ preserves an empirical estimate $\hat{p}_{\text{old}}(x)$ of the previous task's data distribution. In the expectation parameter space ($\boldsymbol{\eta}$-coordinates), this maintains the support of sufficient statistics, preventing the model from drifting entirely out of the moment space of old data during new task training~\cite{amari2016information}.

\paragraph*{EWC as $\theta$-Space Curvature Constraint}
EWC directly constrains movement in natural parameter space ($\theta$-coordinates) via the Fisher Information Matrix $\boldsymbol{F}$ (equivalently $g_{ij}(\theta)$ per Table~\ref{tab:notation}), which serves as the Riemannian metric on $\mathcal{M}$~\cite{martens2020new}. The diagonal elements $F_{kk}$ control the ``curvature'' along each parameter direction; large $F_{kk}$ indicates high sensitivity to changes affecting the old distribution. Geometrically, EWC restricts the model's trajectory to directions orthogonal to the high-curvature axes of the old task.

Thus, replay and EWC form dual regularization mechanisms: replay implicitly constrains the model via \emph{data distribution geometry}, while EWC explicitly constrains via \emph{parameter manifold geometry}. Both define the same feasible region on $\mathcal{M}$---the set of parameters that preserve old task performance---differing only in their operational realization.

\subsubsection*{Heuristic Correspondence: Replay Budget $k$ vs. Regularization Strength $\lambda$}

Consider the empirical Fisher Information computed from $k$ replay samples $\{x_i\}_{i=1}^k \sim p_{\text{old}}$:
\begin{equation}
\hat{F}_{kk}^{(k)} = \frac{1}{k} \sum_{i=1}^k \left( \frac{\partial \log p_{\theta_{\text{old}}^*}(x_i)}{\partial \theta_k} \right)^2.
\label{eq:empirical_fisher}
\end{equation}

When $k$ is limited, the empirical Fisher $\hat{\boldsymbol{F}}^{(k)}$ approximates the true $\boldsymbol{F}$ with estimation noise. This noise effectively \emph{softens} the constraint: smaller $k$ yields higher variance in $\hat{\boldsymbol{F}}^{(k)}$, equivalent to a weaker regularization in $\theta$-space. Conversely, EWC utilizes the exact Fisher (or its diagonal approximation) computed over all old data, with explicit coefficient $\lambda$ controlling constraint strength.

This suggests an \emph{empirical budget correspondence}:
\begin{equation}
k \cdot \lambda_{\text{eff}} \approx C,
\end{equation}
where $C$ depends on the task geometry (e.g., the eigenspectrum of $\boldsymbol{F}$). Under quadratic approximation of the new task loss, the forgetting amount in replay scales as $\text{Tr}(\boldsymbol{F}^{-1}\Sigma_{\text{new}})/k$, while in EWC it scales as $\text{Tr}(\boldsymbol{F}^{-1}\Sigma_{\text{new}})/\lambda$, yielding:
\begin{equation}
\lambda_{\text{eff}} \sim \frac{1}{k} \quad \text{(up to task-dependent scaling)}.
\end{equation}

Thus, the replay buffer size $k$ and the regularization coefficient $\lambda$ act as \emph{dual control variables} for the same geometric constraint: when memory is scarce (small $k$), one must increase $\lambda$ to compensate; when $k$ is large, even mild regularization suffices.

\subsubsection*{Unified Formulation: Constraint on Statistical Manifold}

Both strategies can be unified under the constrained optimization on $\mathcal{M}$:
\begin{equation}
\min_{\theta} \mathcal{L}_{\text{new}}(\theta) \quad \text{s.t.} \quad (\theta - \theta_{\text{old}}^*)^\top \boldsymbol{F}(\theta_{\text{old}}^*) (\theta - \theta_{\text{old}}^*) \leq \epsilon,
\end{equation}
where:
\begin{itemize}
    \item In \textbf{EWC}, $\boldsymbol{F}$ is the exact Fisher, and $\epsilon \propto 1/\lambda$ is determined by the regularization coefficient;
    \item In \textbf{Replay}, $\boldsymbol{F}$ is replaced by the empirical Fisher $\hat{\boldsymbol{F}}^{(k)}$, with effective constraint tightness $\epsilon_{\text{eff}} \sim 1/k$ determined by the buffer size.
\end{itemize}

Therefore, \emph{replay essentially estimates Fisher Information from data samples, while EWC utilizes Fisher Information directly in parameter space}. They represent the same geometric constraint under different observational conditions---finite samples versus exact second-order statistics.

\subsubsection*{Research Directions}

This correspondence suggests several empirically testable hypotheses:
\begin{enumerate}
    \item \textbf{Budget Equivalence Curves}: On standard continual learning benchmarks (e.g., Split-CIFAR or sequential instruction tuning), measure the substitution curve between replay buffer size $k$ and EWC coefficient $\lambda$ to verify $k \cdot \lambda \approx \text{const}$.
    \item \textbf{Hybrid Strategy Optimization}: Combining small-buffer replay (reduced $k$) with moderate EWC (intermediate $\lambda$) may achieve better Pareto fronts (forgetting vs. new task accuracy) than either strategy alone, leveraging the gradient interaction structure of Section~\ref{subsec:gradient-interaction}.
    \item \textbf{Low-Rank Continual Learning}: When parameter updates are constrained to low-rank subspaces (Section~\ref{subsec:low-rank}), does the $k$-vs-$\lambda$ correspondence still hold, or does the rank $r$ introduce a third budget dimension?
\end{enumerate}

\begin{remark}
This correspondence applies primarily to \textbf{Stage~1 (Training)}, requiring gradient access and Fisher Information computation. For extreme scenarios where no old data can be stored (zero-shot continual learning), the duality degenerates to the EWC-only case, but the core insight---sample count as inverse regularization strength---remains valid for approximate forgetting bounds.
\end{remark}

\textbf{Algorithmic Sketch: Fisher-Guided Replay Selection.}
The dual correspondence between replay ($\eta$-space support) and EWC ($\theta$-space curvature) suggests a natural metric for replay sample selection.
Recent empirical work on active and online continual learning has explored gradient‑based or Fisher‑based selection heuristics~\cite{aljundi2019gradient}, yet a unified geometric rationale for \emph{why} Fisher information should govern both parameter protection and data retention has been lacking.
The present framework supplies precisely this missing rationale.
Geometrically, the Fisher information matrix $F$ measures the local curvature of the statistical manifold $\mathcal{M}$ with respect to $\theta$; its trace decomposes into a sum of per‑sample gradient norms $\|\nabla_\theta \ell(x,y)\|^2$.
Consequently, a sample with a large Fisher trace induces a stronger geometric constraint in $\theta$-space, making it an especially informative candidate for replay when the memory budget $k$ is tight.
Formally, given a candidate pool, one may assign each sample $n$ a score $s_n = \|\nabla_\theta \ell(x_n, y_n)\|^2$ and retain the top‑$k$ scoring instances.
This heuristic instantiates the dual‑budget correspondence: a small replay buffer forces the selection of samples that individually exert maximal regularizing pressure on $\theta$-space, thereby compensating for limited $\eta$-space support.
While related selection criteria have appeared independently in the continual learning literature, the geometric framework developed here provides a principled unification, revealing that Fisher‑based selection is not an ad hoc engineering choice but a direct consequence of the Legendre duality between $\eta$‑ and $\theta$‑coordinates.
We leave the systematic empirical validation of this Fisher‑guided replay strategy to future work.

\subsection{Cooperative Enhancement Correspondence: Trajectory Coupling as Geometric Necessity}
\label{subsec:conceptual-security}

Unlike the weak coupling observed in benign fine-tuning (where data selection and parameter updates are largely separable), malicious backdoor attacks reveal a strong adversarial correspondence: data perturbations must precisely induce wide loss basins in parameter space, while parameter updates must solidify the data-imprinted malicious patterns. This bidirectional optimization constitutes the essential structure of data-parameter security duality in modern LLMs.

\begin{figure}[t]
    \centering
    \includegraphics[width=0.95\columnwidth]{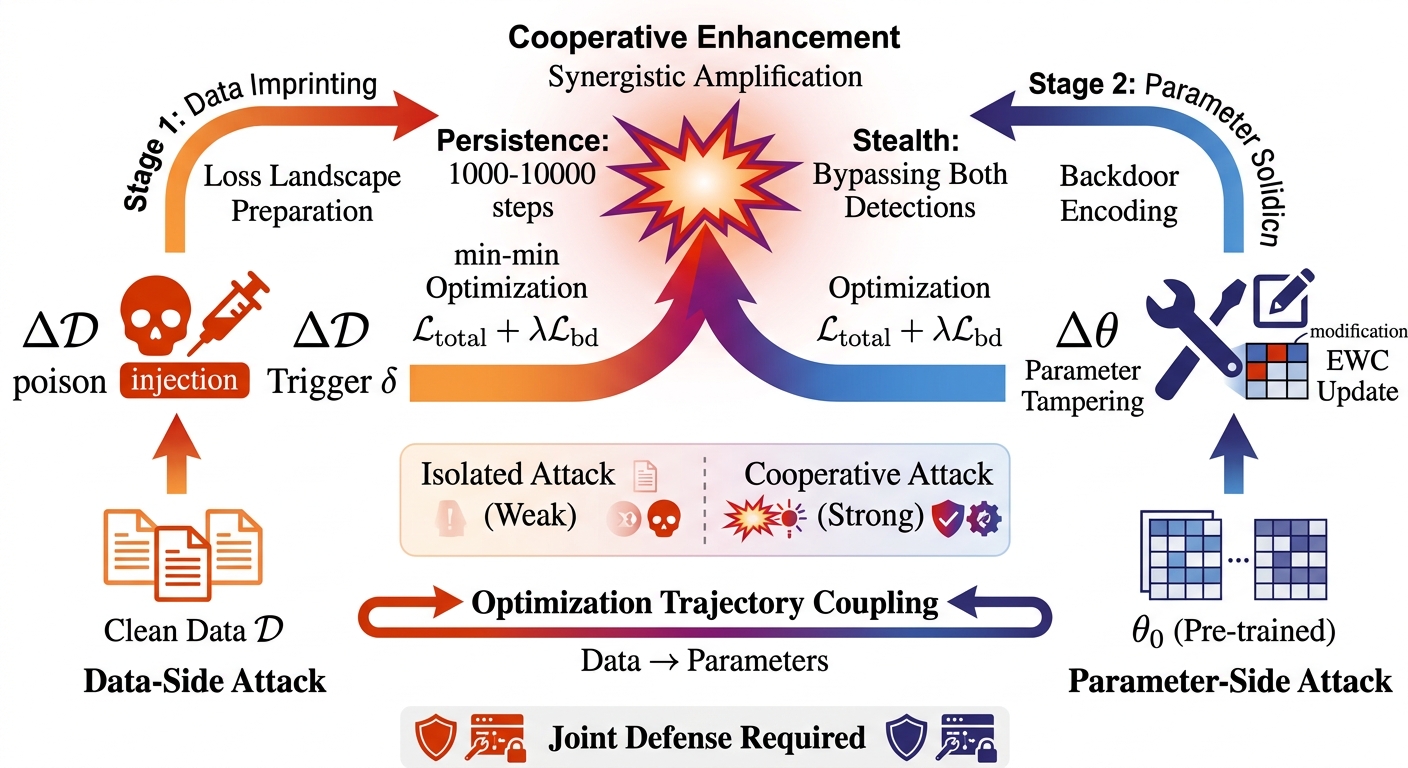}
    \caption{Conceptual Correspondence: Cooperative Data-Parameter Attacks. Unlike benign correspondences where data and parameter operations are separable, adversarial attacks exhibit \emph{cooperative enhancement}: Stage~1 (Data-side) trigger design $\delta$ prepares the loss landscape geometry, while Stage~2 (Parameter-side) solidification via EWC encodes the backdoor into low-curvature regions. The min-min optimization structure creates synergistic amplification where the combined effect (persistent compromise and stealth) exceeds the sum of isolated attacks (which degrade rapidly or exhibit high detectability). This bidirectional dependency necessitates joint defense strategies monitoring both the optimization trajectory and basin geometry.}
    \label{fig:coattack}
\end{figure}

Fig.~\ref{fig:coattack} conceptualizes this adversarial correspondence through the lens of cooperative optimization. \textbf{Stage~1 (Data-side Imprinting)}: The attacker injects poison $\Delta\mathcal{D}$ with trigger $\delta$, preparing the loss landscape geometry (wide basins) via min-min optimization without modifying parameters. \textbf{Stage~2 (Parameter-side Solidification)}: Via EWC, the backdoor is encoded into low-curvature regions of $\theta$-space, crystallizing the malicious pattern into persistent weight configurations that survive 1000--10000 steps of clean fine-tuning. \textbf{Synergistic Amplification}: Unlike benign correspondences where operations are separable, the bidirectional dependency creates non-linear enhancement---the combined effect (persistent compromise + stealth) exceeds the sum of isolated attacks (which degrade rapidly or exhibit high detectability), necessitating joint defense strategies monitoring both optimization trajectory and basin geometry.

\subsubsection{Two-Stage Min-Min Optimization as Geometric Coupling}
\label{sec:two-stage-min-min}

\paragraph{Stage 1: Data-Side Imprinting (Basin Preparation)}
The attacker injects poison $\Delta\mathcal{D}$ with trigger $\delta$, preparing the loss landscape geometry via min-min optimization without modifying parameters. This corresponds to malicious data pruning—selecting and modifying training samples to imprint specific behavioral patterns in $\eta$-space, effectively preparing the ``terrain'' for parameter-space solidification. 

Crucially, this stage induces a specific geometric signature: the poisoned data distribution shifts the empirical mean in $\eta$-space, creating a malicious basin characterized by an induced Fisher curvature $\lambda_{\max}(F_{\text{induced}})$. Without subsequent Stage~2, this basin exhibits high curvature ($\lambda \sim \mathcal{O}(1)$), yielding finite persistence $T_{\text{forget}} \propto 1/\lambda$ under clean fine-tuning.

\paragraph{Stage 2: Parameter-Side Solidification (Curvature Control)}
Via Elastic Weight Consolidation (EWC), the backdoor is encoded into low-curvature regions of $\theta$-space, crystallizing the malicious pattern into persistent weight configurations that survive $10^3$--$10^4$ steps of clean fine-tuning. This corresponds to parameter pruning/masking—compressing the malicious behavior into a persistent subspace while preserving clean-task performance via Fisher-information regularization.

\paragraph{Synergistic Amplification via Gradient Co-Directionality.}
The conceptual correspondence manifests in non-linear synergistic enhancement observed in BadCLIP++: compared to pure data poisoning (which degrades within $<$100 steps) or pure parameter attacks (high detectability), the cooperative approach achieves persistent compromise and stealth simultaneously. 

BadCLIP++'s theoretical analysis establishes that within a trust region around the poisoned optimum, the gradients of clean fine-tuning $\nabla_\theta \mathcal{L}_{\text{clean}}$ and backdoor objectives $\nabla_\theta \mathcal{L}_{\text{mal}}$ are \emph{co-directional}~\cite{badclipplusplus2026}. This co-directionality implies that the backdoor basin and the clean-task basin share overlapping tangent cones—a geometric condition that prevents clean gradient steps from escaping the malicious attractor. 

The attack persistence $T_{\text{persist}}$—the number of clean fine-tuning steps before ASR degrades—is thus governed not by a simple curvature product, but by the \emph{angular alignment} $\theta_{\text{align}}$ between the clean and backdoor gradients:
\begin{equation}
T_{\text{persist}} \propto \frac{1}{\arccos\left( \frac{\langle \nabla \mathcal{L}_{\text{clean}}, \nabla \mathcal{L}_{\text{mal}} \rangle}{\|\nabla \mathcal{L}_{\text{clean}}\| \|\nabla \mathcal{L}_{\text{mal}}\|} \right)} = \frac{1}{\theta_{\text{align}}}.
\label{eq:persistence-angle}
\end{equation}
When $\theta_{\text{align}} \to 0$ (perfect co-directionality), $T_{\text{persist}} \to \infty$, explaining why BadCLIP++ maintains $>99\%$ ASR after extensive fine-tuning. This angle-based characterization provides a geometrically rigorous and empirically testable alternative to heuristic curvature-product conjectures.

\subsubsection{Implications for Joint Defense: Fisher Information as the Bridge Metric}
\label{sec:joint-defense}

The trajectory coupling dictates that unilateral defenses—either data sanitization ($\mathcal{P}_d$) or parameter auditing ($\mathcal{P}_p$)—are fundamentally insufficient against coordinated attacks. Effective defense must mirror the attacker's dual-stage strategy by monitoring the \emph{optimization trajectory continuity} across the Stage~1$\to$Stage~2 transition.

Recent defense methods have independently converged on Fisher information as the unifying diagnostic tool that bridges $\eta$-space (data) and $\theta$-space (parameters):

\paragraph{Stage 1 Defense: Disrupting Basin Preparation via FIP}
Fisher Information guided Purification (FIP)~\cite{karim2024fip} serves as a paradigmatic Stage~1 defense. FIP leverages the observation that backdoored models typically converge to \emph{sharper minima} in the loss landscape—precisely the geometric signature of Stage~1 basin preparation where $\lambda_{\max}(F_{\text{induced}})$ is artificially inflated. By utilizing the Fisher Information Matrix to design regularization terms that re-optimize the model toward flatter minima, FIP effectively disrupts the ``wide basin'' prepared by data poisoning, causing the backdoor to dissipate under subsequent clean fine-tuning. This corresponds to monitoring $\eta$-space for anomalous Fisher trace elevations that indicate malicious basin preparation.

\paragraph{Stage 2 Defense: Detecting Parameter Solidification via FDCR}
Parameter Disparities Dissection via Fisher Discrepancy Cluster and Rescale (FDCR)~\cite{huang2024fdcr} addresses Stage~2 solidification in federated learning scenarios. FDCR computes parameter importance using Fisher information to identify clients whose updates exhibit anomalous importance patterns—statistical signatures of Stage~2 parameter tampering where backdoor logic has been crystallized into specific parameter subsets via curvature control. By reweighting client updates based on Fisher discrepancy, FDCR effectively audits $\theta$-space for the low-curvature, high-persistence configurations characteristic of solidified backdoors.

\paragraph{Unified Defense Protocol}
The geometric framework developed here provides the theoretical rationale for why Fisher information serves as the natural bridge metric: it simultaneously encodes $\eta$-space sample geometry (via per-sample gradient norms $\|\nabla_\theta \ell(x;\theta)\|^2$) and $\theta$-space parameter sensitivity (via the Fisher metric $g_{ij}(\theta)$). A principled joint defense would thus monitor the \emph{continuity of Fisher information flow}:

\begin{enumerate}
    \item \textbf{Stage 1 Detection (Data-side):} Monitor for anomalous flattening of the Fisher trace $\operatorname{Tr}(F)$ during micro-training probes. Specifically, compute the Fisher trace ratio for candidate batch $\mathcal{B}$:
    \begin{equation}
    \rho_{\text{flat}}(\mathcal{B}) = \frac{\mathbb{E}_{x\in\mathcal{B}}[\|\nabla_\theta \ell(x;\theta)\|^2]}{\mathbb{E}_{x\in\mathcal{D}_{\text{clean}}}[\|\nabla_\theta \ell(x;\theta)\|^2]}.
    \label{eq:fisher-trace-ratio}
    \end{equation}
    Suspiciously low values $\rho_{\text{flat}} \ll 1$ indicate that the batch induces a substantially flatter basin than clean data—a necessary geometric precondition for persistent backdoor implantation (Stage~1 preparation).
    
    \item \textbf{Stage 2 Detection (Parameter-side):} Enforce curvature regularization via Fisher discrepancy analysis to identify solidified backdoors. Compute the per-parameter Fisher importance $F_{kk} = \mathbb{E}[(\partial \log p_\theta(x)/\partial \theta_k)^2]$ and flag parameters exhibiting anomalous importance concentration $\frac{\max_k F_{kk}}{\operatorname{Tr}(F)} > \tau_{\text{solidify}}$, indicating Stage~2 solidification.
\end{enumerate}

\subsubsection{Testable Hypotheses and Geometric Predictions}
\label{sec:testable-hypotheses}

The cooperative enhancement correspondence yields the following empirically falsifiable predictions:

\begin{enumerate}[leftmargin=*]
    \item \textbf{Co-directionality Hypothesis:} The attack persistence $T_{\text{persist}}$ is monotonically increasing with the cosine similarity between clean and backdoor gradients measured at the poisoned optimum (Eq.~\ref{eq:persistence-angle}). This can be validated by measuring gradient alignment in BadCLIP++ checkpoints at varying EWC regularization strengths.
    
    \item \textbf{Fisher Trace Hypothesis:} Poisoned datasets $\mathcal{D}_{\text{bd}}$ exhibit significantly lower Fisher trace ratios $\rho_{\text{flat}}$ (Eq.~\ref{eq:fisher-trace-ratio}) compared to clean datasets during the initial $K$ steps of training ($K \approx 10$--$20$). This predicts that FIP's purification effectiveness correlates with $\rho_{\text{flat}}^{-1}$.
    
    \item \textbf{Defense Sufficiency Hypothesis:} A defense that monitors \emph{both} $\eta$-space Fisher trace anomalies (Stage~1) and $\theta$-space Fisher discrepancy (Stage~2) achieves strictly lower false-negative rates than either unilateral defense alone. This can be evaluated by combining FIP (Stage~1) and FDCR (Stage~2) in a pipeline.
\end{enumerate}

\subsection{Privacy Correspondence: Information-Theoretic Duality in Data and Parameter Spaces}
\label{subsec:privacy-correspondence}

\textbf{Relation to Differential Privacy Composition.}
Our cascading privacy bound $\epsilon_{\text{eff}} \leq \epsilon_d + \rho \cdot \epsilon_p$ operates in a distinct regime from standard DP composition theorems \cite{dwork2014algorithmic}. While sequential composition yields additive accumulation $\epsilon_{\text{total}} = \sum_i \epsilon_i$ for mechanisms applied to the same dataset, our framework considers a \textit{two-stage pipeline}: data-side sufficient statistic extraction (reducing dimension from $d$ to $m = \rho d$) followed by parameter-side gradient compression. The compression ratio $\rho$ acts as an \textit{information-theoretic attenuation factor} rather than a privacy parameter per se. This is conceptually related to \textit{privacy amplification by subsampling} \cite{kasiviswanathan2011can,wang2019subsampled}, where analyzing random subsets strengthens privacy guarantees. In our setting, aggressive data compression ($\rho \ll 1$) reduces the effective information available to parameters, thereby attenuating---not accumulating---the parameter-side privacy cost. We emphasize that $\epsilon_{\text{eff}} \leq \epsilon_d + \rho \cdot \epsilon_p$ is an information-theoretic bound on mutual information $I(\mathcal{D}; \theta)$, not a formal $(\epsilon, \delta)$-DP guarantee; translating our cascading structure to rigorous DP composition theorems remains for future work.

While Section~\ref{subsec:conceptual-security} addresses adversarial \emph{integrity} threats (where attackers inject malicious perturbations via $\Delta\mathcal{D}$ or $\Delta\theta$), we now examine the \emph{confidentiality} correspondence governing information leakage between data and parameter domains. Contemporary large language models (LLMs) exhibit unprecedented vulnerability to training data extraction~\cite{nasr2025scalable} and parameter inversion attacks~\cite{zhang2025parameter}, revealing that data privacy and parameter privacy are not independent security objectives but structurally coupled through the mutual information channel $I(\mathcal{D};\theta)$. Unlike the cooperative enhancement observed in attacks, privacy protection exhibits a \emph{cascading dependency}: safeguarding data privacy inherently constrains parameter inference risk, and vice versa, through shared information-theoretic bounds.

\begin{figure}[t]
    \centering
    \includegraphics[width=0.95\columnwidth]{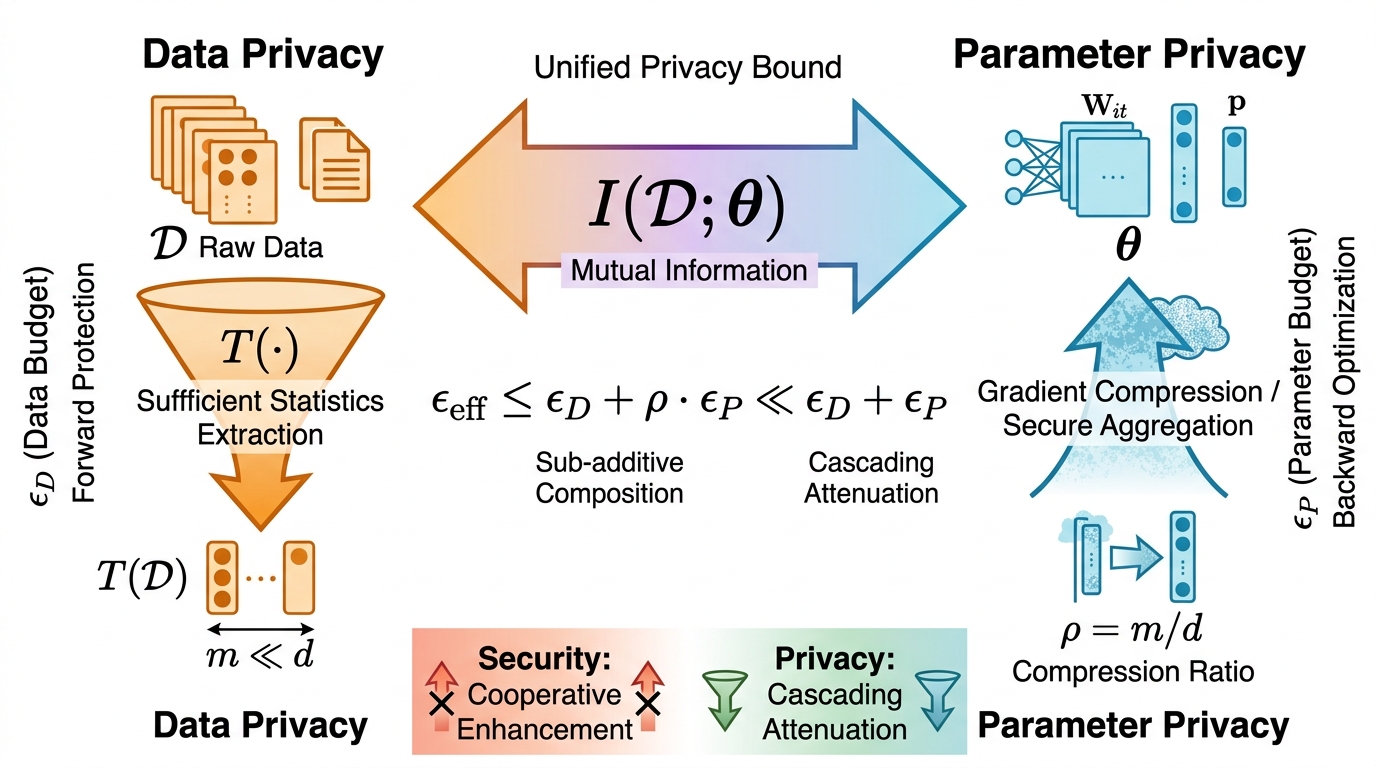}
    \caption{Privacy Correspondence: Information-Theoretic Duality. The mutual information $I(\mathcal{D};\theta)$ serves as a unified upper bound for both data-to-parameter leakage ($\epsilon_d$) and parameter-to-data reconstruction ($\epsilon_p$). Data-side sufficient statistic extraction ($T(\mathcal{D})$) provides forward protection via dimensionality reduction (compression ratio $\rho = m/d$), while parameter-side gradient compression enables backward optimization. The total privacy cost follows a sub-additive cascading composition $\epsilon_{\text{eff}} \leq \epsilon_d + \rho \cdot \epsilon_p \ll \epsilon_d + \epsilon_p$, revealing that data compression yields multiplicative returns in parameter privacy. This cascading attenuation contrasts with the cooperative enhancement observed in adversarial security (Section~\ref{subsec:conceptual-security}), where perturbations amplify rather than attenuate across the data-parameter boundary.}
    \label{fig:privacy-correspondence}
\end{figure}

Fig.~\ref{fig:privacy-correspondence} illustrates the information-theoretic duality governing privacy protection across the data-parameter boundary. \textbf{Left (Data Privacy)}: Raw data $\mathcal{D}$ undergoes sufficient statistic extraction $T(\cdot)$, reducing dimensionality from $d$ to $m$ (compression ratio $\rho = m/d \ll 1$) to provide forward protection with budget $\epsilon_d$. \textbf{Right (Parameter Privacy)}: Parameters $\theta$ are protected via gradient compression and secure aggregation with budget $\epsilon_p$, constraining backward reconstruction risk. \textbf{Center}: Mutual information $I(\mathcal{D};\theta)$ serves as the unified upper bound for both data-to-parameter leakage ($\epsilon_d$) and parameter-to-data reconstruction ($\epsilon_p$), reflecting the fundamental symmetry $I(\mathcal{D};\theta) = I(\theta;\mathcal{D})$. \textbf{Cascading Attenuation}: The total privacy cost follows sub-additive composition $\epsilon_{\text{eff}} \leq \epsilon_d + \rho \cdot \epsilon_p \ll \epsilon_d + \epsilon_p$, revealing that aggressive data-side compression ($\rho \downarrow$) yields multiplicative returns in parameter privacy---directly contrasting with the cooperative amplification observed in adversarial attacks (Section~\ref{subsec:conceptual-security}).

\paragraph{Unified Mutual Information Bound.}
Recent advances in information-theoretic privacy~\cite{wang2025information, li2025sufficient} establish that the fundamental symmetry of mutual information $I(\mathcal{D};\theta) = I(\theta;\mathcal{D})$ serves as a unified upper bound for both data-to-parameter leakage ($\mathcal{L}_D$) and parameter-to-data reconstruction ($\mathcal{L}_{P\to D}$):
\begin{equation}
\mathcal{L}_D \leq I(\mathcal{D};\theta), \quad \mathcal{L}_{P\to D} \leq I(\mathcal{D};\theta).
\end{equation}
This framework breaks the traditional methodological fragmentation where data privacy is measured via $(\epsilon,\delta)$-differential privacy~\cite{dwork2014algorithmic} and parameter privacy via membership inference accuracy~\cite{shokri2017membership}. Instead, $I(\mathcal{D};\theta)$ quantifies the maximal information flow across the data-parameter boundary, regardless of direction, providing a \emph{correspondence metric} for joint privacy accounting~\cite{guo2025cascading}.

\paragraph{Bidirectional Constraint Propagation.}
The privacy correspondence manifests through bidirectional constraint propagation during the training pipeline:

\textit{Data $\to$ Parameter (Forward Protection).} Privacy-preserving operations on the data side (e.g., sufficient statistic extraction $T(\mathcal{D})$) directly constrain the information available to parameters. Following the information bottleneck principle~\cite{li2025sufficient}, we have:
\begin{equation}
I(\mathcal{D};\theta) \leq I(T(\mathcal{D});\theta) + I(\mathcal{D};\theta\,|\,T(\mathcal{D})) \approx I(T(\mathcal{D});\theta),
\end{equation}
where the residual term vanishes for sufficient statistics. Thus, data-side dimensionality reduction \emph{pre-consumes} the privacy budget, reducing leakage risk through subsequent parameter exposure.

\textit{Parameter $\to$ Data (Backward Optimization).} Conversely, parameter-side constraints (e.g., gradient compression~\cite{duan2025federated}, secure aggregation~\cite{bonawitz2017practical}) induce geometric constraints on the feasible data reconstruction space. The optimal noise configuration $\sigma_{\text{opt}}$ for parameter protection implicitly determines the compression rate $\rho$ required on the data side, creating a feedback loop for joint optimization of the total privacy budget.

\paragraph{Structural Symmetry in Privacy-Utility Trade-offs.}
Data privacy and parameter privacy exhibit topological similarity (homeomorphism) in their privacy-utility Pareto fronts~\cite{wang2025information}. Consider the constrained optimization problems:
\begin{align}
\text{Data Privacy:} \quad &\min_{\mathcal{A}} \mathbb{E}[\mathcal{U}(\theta)] \quad \text{s.t.} \quad I(\mathcal{D};\theta) \leq \epsilon_d, \\
\text{Parameter Privacy:} \quad &\min_{\mathcal{A}} \mathbb{E}[\mathcal{U}(\theta)] \quad \text{s.t.} \quad \mathcal{L}_{P\to D}(\theta) \leq \epsilon_p,
\end{align}
where $\mathcal{A}$ denotes the training algorithm and $\mathcal{U}(\theta)$ the model utility. Under comparable privacy budgets ($\epsilon_d \approx \epsilon_p$), these optimizations yield structurally similar Pareto fronts, indicating that the fundamental privacy-utility dilemma is \emph{invariant} to whether the constraint is applied at the data ingress or parameter egress. This symmetry enables \emph{unified system design}: mechanisms developed for data transmission (e.g., secure multi-party computation~\cite{yao1982protocols}) admit direct structural analogues for parameter communication.

\paragraph{Cascading Optimization and Sub-additive Composition.}
A salient implication of the privacy correspondence is that the total privacy cost follows a \emph{sub-additive} (cascading) composition enabled by the information bottleneck~\cite{guo2025cascading}:
\begin{equation}
\epsilon_{\text{eff}} \leq \epsilon_d + \rho \cdot \epsilon_p \ll \epsilon_d + \epsilon_p,
\end{equation}
where $\rho = m/d \ll 1$ represents the compression ratio between the sufficient statistic dimension ($m$) and the original data dimension ($d$), per Table~\ref{tab:notation}.

This reveals an \emph{exchangeability} between data compression and privacy budget: aggressive data-side dimensionality reduction (low $\rho$) geometrically attenuates the parameter-side privacy requirement. Such cascading optimization resolves the ``double taxation'' dilemma in federated learning~\cite{duan2025federated}, where naive composition of local DP ($\epsilon_d$) and global DP ($\epsilon_p$) previously led to catastrophic utility degradation. The correspondence thus provides a constructive framework for \emph{privacy resource allocation}: investing computational resources in data-side sufficient statistic extraction yields multiplicative returns in parameter-side privacy protection~\cite{li2025sufficient}.

\paragraph{Duality with Security Correspondence.}
The privacy correspondence exhibits a fundamental asymmetry with the security correspondence (Section~\ref{subsec:conceptual-security}): while security attacks leverage \emph{cooperative enhancement} (where data and parameter perturbations synergistically amplify malicious effects), privacy protection leverages \emph{cascading attenuation} (where data-side compression cascades into parameter-side leakage reduction). The former represents a vulnerability to be mitigated; the latter represents a structural property to be exploited. Both, however, confirm the central thesis: data and parameter domains are not isolated computational stages, but \emph{coupled information channels} where operations on one side inexorably constrain the geometry of the other.

\subsection{Testing-Time Bidirectional Security Correspondence}
\label{subsec:testing-defense}

While Section~\ref{subsec:conceptual-security} characterizes the \emph{offensive} correspondence via cooperative enhancement between data poisoning $\Delta\mathcal{D}$ and parameter manipulation $\Delta\theta$, we now establish the \emph{defensive} correspondence under frozen-weight constraints (Stage~3 in Table~\ref{tab:stages}). Distinct from the geometric equivalence in Section~\ref{subsec:jacobian-framework} that focuses on functional approximation via Jacobian image spaces, this correspondence targets robustness certification: the input perturbation budget $\epsilon_d$ (data-side) and the local Lipschitz constant of the forward mapping (parameter-side) form a multiplicative safety margin that governs autoregressive generation security.

\begin{figure}[t]
    \centering
    \includegraphics[width=0.95\columnwidth]{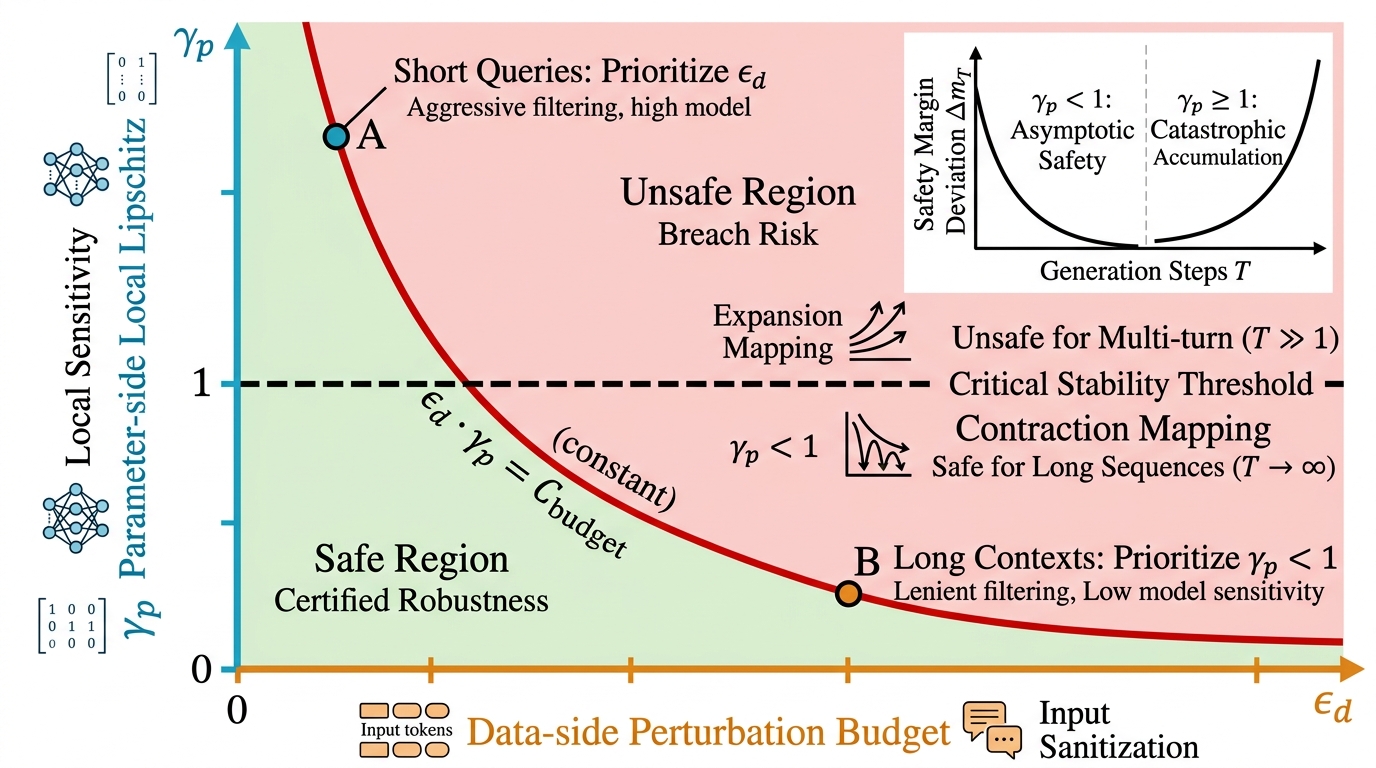}
    \caption{Testing-Time Security Correspondence: The product constraint $\epsilon_d \cdot \gamma_p \leq C_{\text{budget}}$ governs the certified safety margin, where data-side input sanitization budget $\epsilon_d$ and parameter-side local Lipschitz constant $\gamma_p$ form a multiplicative coupling. The critical threshold $\gamma_p = 1$ separates contraction mappings (asymptotic safety for long sequences, green region) from expansion mappings (catastrophic error accumulation, red region), dictating distinct optimal allocation strategies for short-form queries (prioritize $\epsilon_d$) versus long-form multi-turn contexts (prioritize $\gamma_p < 1$).}
    \label{fig:testsecurity}
\end{figure}

Fig.~\ref{fig:testsecurity} illustrates the defensive correspondence under frozen-weight constraints (Stage~3). The certified safety margin is governed by the product constraint $\epsilon_d \cdot \gamma_p \leq C_{\text{budget}}$, where $\epsilon_d$ represents the data-side input sanitization budget (horizontal axis) and $\gamma_p$ denotes the parameter-side local Lipschitz constant\footnote{We use $\gamma_p$ (sensitivity threshold) to distinguish from $\epsilon_p$ (parameter perturbation magnitude in Table~\ref{tab:notation}), as no explicit parameter updates occur in Stage~3.} (vertical axis). The critical threshold $\gamma_p = 1$ (dashed line) delineates two fundamental regimes: \textbf{contraction mappings} ($\gamma_p < 1$, green region) ensuring asymptotic safety for long sequences via exponential error decay, and \textbf{expansion mappings} ($\gamma_p \geq 1$, red region) leading to catastrophic accumulation of adversarial deviations. This phase transition dictates distinct optimal allocation strategies: for short queries (Point~A), resources prioritize aggressive input filtering (tight $\epsilon_d$) with moderate model sensitivity; for long-form multi-turn contexts (Point~B), enforcing $\gamma_p < 1$ becomes necessary and sufficient, permitting looser input sanitization to preserve semantic utility while guaranteeing multi-turn stability.

\subsubsection{Preliminaries and Threat Model}
At inference, the model parameters are frozen at $\theta^* \in \Theta$ (typically $\theta^* \approx \theta_0$ for pre-trained models or adapted weights from Stage~2 compression). The input token space is $\mathcal{X}$, with embedding function $\mathrm{Embed}: \mathcal{X} \to \mathcal{E} \subset \mathbb{R}^d$. The single-step forward mapping is $F_{\theta^*}: \mathcal{E} \to \mathcal{H}$, where $\mathcal{H} \subseteq \mathbb{R}^m$ denotes the hidden state space. We adopt the observable safety margin $S: \mathcal{H} \to \mathbb{R}$ (e.g., the logit difference between harmless and harmful classes in a safety classifier) as the sole security metric, avoiding intractable implicit safety sets in $\mathcal{H}$.

The threat model encompasses inference-time evasion attacks (adversarial prompts) and prompt injection, which induce perturbations $\Delta x$ in the embedding space. The attacker's budget is constrained by $\|\Delta x\|_2 \leq \epsilon_d$, where $\epsilon_d$ corresponds to the data-space perturbation budget defined in Section~\ref{subsec:adversarial}. The defender aims to guarantee $S(h) \geq \tau_{\mathrm{safe}}$ for all allowable perturbations without updating $\theta^*$.

\subsubsection{Dual Security Constraints}
The defensive correspondence operates via two coupled constraints that mirror the data-parameter duality in Stage~3:

\paragraph{Data-Side Constraint (Input Sanitization)}
The defense restricts the input to a safe neighborhood in the embedding space:
\begin{equation}
\|\mathrm{Embed}(x) - \boldsymbol{\mu}_e\|_2 \leq \epsilon_d,
\label{eq:data-constraint}
\end{equation}
where $\boldsymbol{\mu}_e = \mathrm{Embed}(\boldsymbol{\mu}_x)$ represents the embedding centroid of benign inputs. This constraint is enforced through input filtering, adversarial detection, or prompt purification~\cite{nasr2025scalable}, without modifying $\theta^*$.

\paragraph{Parameter-Side Constraint (Local Sensitivity)}
Without gradient backpropagation to weights, the defense constrains the model's local sensitivity via the Lipschitz constant of the forward mapping:
\begin{equation}
\|\nabla_e F_{\theta^*}(e)\|_{\mathrm{op}} \leq \gamma_p, \quad \forall e \in \mathcal{B}(\boldsymbol{\mu}_e, \epsilon_d),
\label{eq:param-constraint}
\end{equation}
where $\|\cdot\|_{\mathrm{op}}$ denotes the operator norm (largest singular value) and $\gamma_p$ represents the parameter-side local Lipschitz bound. This is enforced through inference-time interventions such as attention-head suppression, hidden-state regularization, or dynamic quantization~\cite{zhang2025parameter}, which effectively reduce the Jacobian magnitude without altering $\theta^*$.

\subsubsection{Unified Product Constraint}
Combining Eq.~\eqref{eq:data-constraint} and Eq.~\eqref{eq:param-constraint} via the composition of Lipschitz functions yields the \emph{product safety criterion}. Assuming $S$ is $C$-Lipschitz continuous with respect to $h$, the safety margin at hidden state $h = F_{\theta^*}(e)$ satisfies:
\begin{equation}
S(h) \geq S_0 - C \cdot \epsilon_d \cdot \gamma_p,
\label{eq:product-constraint}
\end{equation}
where $S_0 = S(F_{\theta^*}(\boldsymbol{\mu}_e))$ is the baseline safety margin for benign inputs.

This reveals that data-side and parameter-side defenses are not additive but \emph{multiplicatively coupled}: tightening only $\epsilon_d$ (aggressive input filtering) while leaving $\gamma_p$ large provides diminishing returns, and vice versa. The product $\epsilon_d \cdot \gamma_p$ must be minimized to maximize the certified safety margin. This structure mirrors the cascading privacy budget in Section~\ref{subsec:privacy-correspondence}, but operates on robustness certification rather than information-theoretic bounds.

\subsubsection{Dynamic Stability in Autoregressive Generation}
\label{sec:dynamic-stability}

For autoregressive LLMs, security must be certified over $T$ generation steps. Strictly distinguishing single-turn autoregressive generation from multi-turn dialogue, the recursive hidden-state update is:
\begin{equation}
h_{t+1} = G_{\theta^*}(h_t), \quad t = 0, 1, \ldots, T-1,
\end{equation}
where $G_{\theta^*}$ represents the single-token generation mapping with frozen parameters, and $h_0 = F_{\theta^*}(e_0)$ for initial embedding $e_0$. Under the local Lipschitz condition $\mathrm{Lip}(G_{\theta^*}) \leq \gamma_p$, and assuming no additional external perturbations during generation (pure autoregressive closed-loop), the safety margin deviation $\Delta m_t = S_0 - S(h_t)$ evolves as:
\begin{equation}
\Delta m_{t+1} \leq \gamma_p \cdot \Delta m_t.
\end{equation}

With initial deviation $\Delta m_0 \leq C \epsilon_d \gamma_p$ from Eq.~\eqref{eq:product-constraint}, the accumulated deviation after $T$ steps satisfies:
\begin{equation}
\Delta m_T \leq C \epsilon_d \gamma_p^{T+1}.
\label{eq:accumulation}
\end{equation}

\paragraph{Critical Stability Condition}
Eq.~\eqref{eq:accumulation} reveals a sharp phase transition for long-form generation:
\begin{itemize}
    \item If $\gamma_p < 1$ (\textbf{contraction mapping}), $\Delta m_T \to 0$ exponentially as $T \to \infty$, ensuring asymptotic safety regardless of conversation length;
    \item If $\gamma_p \geq 1$ (\textbf{expansion mapping}), the deviation grows exponentially with $T$, inevitably breaching the safety threshold for long-form generation, irrespective of how strictly $\epsilon_d$ is controlled.
\end{itemize}
This establishes $\gamma_p < 1$ as a \emph{necessary and sufficient} condition for multi-turn dialogue security, rigorously explaining why input filtering alone fails against long-context jailbreaking attacks~\cite{jain2023baseline}.

For multi-turn dialogue with $K$ user interventions at steps $\{t_k\}_{k=1}^K$, each introducing fresh perturbation bounded by $\epsilon_d$, the deviation bound generalizes to:
\begin{equation}
\Delta m_T \leq C \epsilon_d \sum_{k=0}^{K} \gamma_p^{T-t_k},
\end{equation}
where $t_0 = 0$ and $t_k$ is the step index of the $k$-th external input. This confirms that without the contraction condition $\gamma_p < 1$, error accumulates catastrophically across turns.

\subsubsection{Security Budget Allocation}
The product constraint $\epsilon_d \cdot \gamma_p \leq C_{\text{budget}}$ implies a fundamental trade-off between input sanitization cost and inference-time regularization overhead. Formulating the optimal allocation as a convex optimization problem with fixed total budget $C$:
\begin{align}
\min_{\epsilon_d, \gamma_p} \quad & \mathcal{L}_d(\epsilon_d) + \mathcal{L}_p(\gamma_p) \\
\text{s.t.} \quad & \epsilon_d \cdot \gamma_p \leq C_{\text{budget}}, \nonumber \\
& \epsilon_d, \gamma_p > 0, \nonumber
\end{align}
where $\mathcal{L}_d$ represents semantic utility loss (tightening $\epsilon_d$ may filter benign content or reduce semantic diversity) and $\mathcal{L}_p$ represents generation quality degradation (reducing $\gamma_p$ suppresses model capability and fluency). 

Assuming convex loss functions $\mathcal{L}_d(\epsilon_d) = k_d/\epsilon_d$ and $\mathcal{L}_p(\gamma_p) = k_p/\gamma_p$ reflecting the marginal cost of constraint tightening, the KKT conditions yield the optimal allocation principle:
\begin{equation}
\frac{\partial \mathcal{L}_d}{\partial \epsilon_d} \cdot \epsilon_d = \frac{\partial \mathcal{L}_p}{\partial \gamma_p} \cdot \gamma_p,
\end{equation}
indicating that resources should be allocated such that the marginal cost of each side is proportional to its constraint tightness. 

\paragraph{Operational Implications}
This allocation yields context-dependent defense strategies:
\begin{itemize}
    \item \textbf{Short-form, single-turn queries} (small $T$): Prioritize tightening $\epsilon_d$ (aggressive input filtering) while allowing moderate $\gamma_p$, as the exponential term $\gamma_p^{T+1}$ remains controlled;
    \item \textbf{Long-form or multi-turn contexts} (large $T$): Prioritize reducing $\gamma_p$ to ensure $\gamma_p < 1$ (contraction), accepting looser $\epsilon_d$ to preserve semantic utility, as input filtering alone cannot prevent error accumulation over many steps.
\end{itemize}

\subsubsection{Relation to Broader Correspondences}
This testing-time security correspondence complements the geometric framework of Section~\ref{subsec:jacobian-framework}: while the Jacobian Image Space characterizes the \emph{feasible directions} of data-parameter duality via $\operatorname{Im}(J_x)$ and $\operatorname{Im}(J_\theta)$, the Lipschitz constraint quantifies the \emph{magnitude} of allowable perturbations within those directions. 

Furthermore, it forms a defensive dual to the offensive cooperative enhancement in Section~\ref{subsec:conceptual-security}: attackers exploit the cooperative optimization of $\Delta\mathcal{D}$ and $\Delta\theta$ to amplify malicious behavior (min-min structure), while defenders must constrain the product of data perturbation and parameter sensitivity to certify safety (constrained min-max structure). Together, these correspondences complete the security landscape across the model lifecycle: Stage~1 (training) features gradient-based alignment, Stage~2 (post-training) permits parameter calibration, and Stage~3 (inference) requires the Lipschitz-based bidirectional constraints established here.

\begin{remark}
The contraction condition $\gamma_p < 1$ can be empirically verified via power iteration on the Jacobian-vector product $\nabla_h G_{\theta^*}(h) \cdot v$ without full matrix materialization, or approximated via attention-head activation variance and layer-wise gradient norms~\cite{sokolic2017robust}, ensuring the framework remains computationally viable for deployed LLMs.
\end{remark}

\textbf{Empirical Grounding}
The product constraint, though derived in continuous embedding space, admits empirical grounding through recent advances in Lipschitz estimation and continuous adversarial training. The local Lipschitz constant $\gamma_p$ of transformer blocks can be bounded and regularized---the JaSMin method explicitly minimizes the Jacobian softmax norm to decrease local Lipschitz constants and boost robustness~\cite{yudin2025attention}. Concurrently, continuous adversarial training (CAT) on LLMs provably improves jailbreak robustness, with generalization bounds that correlate negatively with the embedding-space perturbation radius~\cite{fu2026understanding}. These results substantiate the two orthogonal axes of the product constraint: $\gamma_p$ can be measured and controlled, and $\epsilon_d$ meaningfully captures the magnitude of embedding-space attacks. Moreover, the construction of generalized jailbreak vectors as latent-space perturbations $v_j^l$ demonstrates that attack success hinges on precisely the product of perturbation norm and local model sensitivity---the geometric essence of $\epsilon_d \cdot \gamma_p$.

\textbf{Relation to the Defense Trilemma}
The product constraint derived above offers a quantitative complement to the recent \emph{Defense Trilemma} theorem~\cite{bhatt2026defense}, which proves that no continuous, utility-preserving input wrapper can guarantee complete safety for LLMs with connected prompt spaces. Bhatt et al.\ establish an impossibility result---continuity, utility preservation, and completeness cannot coexist---and identify a Lipschitz condition $G > \ell(K+1)$ (where $K$ is a local Lipschitz bound of the model and $G$ that of the defense) under which a positive-measure set of inputs remains unsafe. Their topological analysis answers \emph{whether} defenses must fail; our product constraint $\epsilon_d \cdot \gamma_p \le C_{\text{budget}}$ answers \emph{how large} the vulnerable region can become. When $\epsilon_d \cdot \gamma_p$ exceeds the safety budget, the conditions for the Trilemma's persistent unsafe region are met. The two frameworks are thus complementary: topological continuity establishes the inevitability of failure, while Lipschitz geometry quantifies the margin of vulnerability.

\subsection{Composition Correspondence: Data Mixing vs. Model Merging}
\label{subsec:composition}

\textbf{Relation to Linear Mode Connectivity.}
A distinct line of research---\textit{Linear Mode Connectivity} (LMC)~\cite{frankle2020linear,entezari2022role,ainsworth2023git}---investigates conditions under which fine-tuned models can be linearly interpolated in parameter space while maintaining low loss. LMC addresses the \textit{feasibility} of model merging, i.e., the geometry of loss barriers along interpolation paths. Our framework addresses a \textit{different question}: given that merging is feasible (models occupy the same loss basin), why do data mixing (in $\eta$-space) and model merging (in $\theta$-space) constitute dual operations? This duality arises because both operations can be formulated as finding a \textit{Bregman barycenter} on the statistical manifold $\mathcal{M}$---data mixing as a barycenter with respect to the Bregman divergence $D_\phi$ in $\eta$-coordinates, and model merging as a barycenter with respect to $D_\psi$ in $\theta$-coordinates. The shared pre-training assumption underlying LMC ensures models reside in a local neighborhood on $\mathcal{M}$ where the Bregman barycenter is well-defined; this is a \textit{precondition} for our analysis, not its object.

\begin{figure}[t]
    \centering
    \includegraphics[width=0.95\columnwidth]{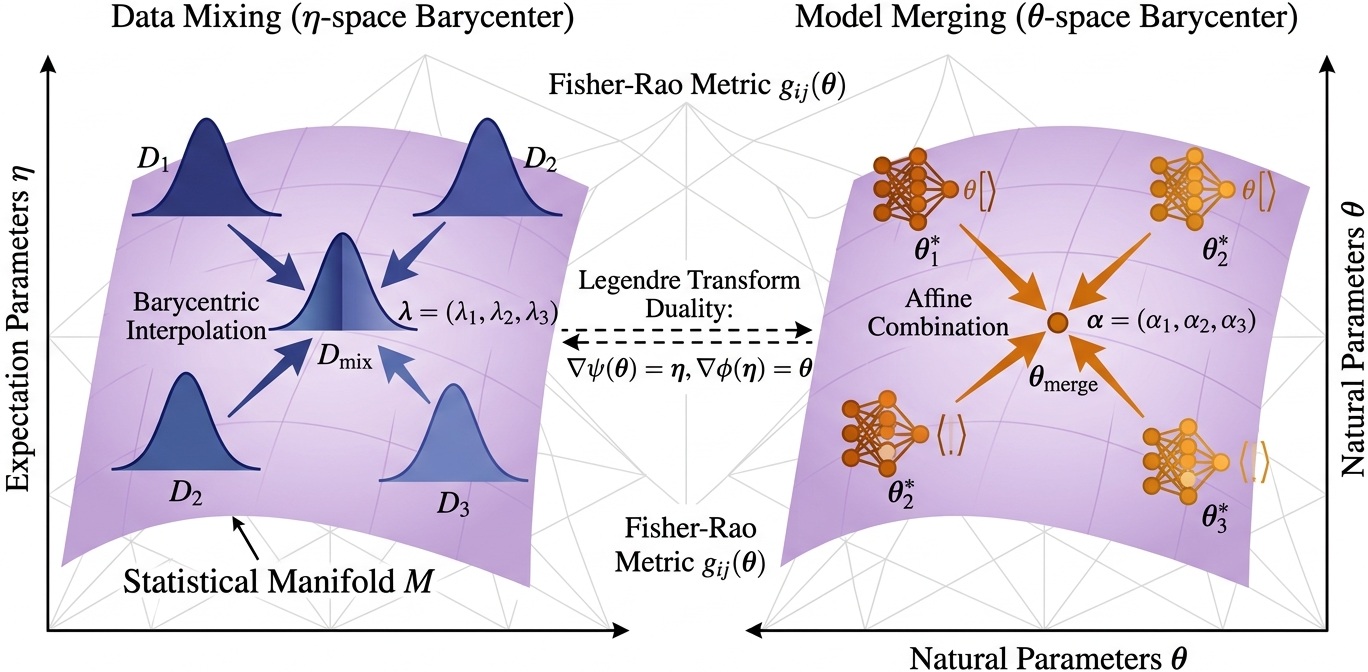}
    \caption{The Composition Correspondence: Multiple data distributions $\{\mathcal{D}_i\}$ (left) are mixed via barycentric interpolation in $\boldsymbol{\eta}$-space, while corresponding model parameters $\{\theta_i^*\}$ (right) are merged via affine combination in $\theta$-space. Both operations converge to Bregman barycenters on the statistical manifold $\mathcal{M}$, which are dual to each other under the Legendre transformation connecting $D_\phi$ and $D_\psi$. The optimal mixing coefficients $\lambda$ and merging coefficients $\alpha$ constitute dual coordinates under the Fisher-Rao metric $g_{ij}(\theta)$.}
    \label{fig:mixing}
\end{figure}

\textbf{Boundary with existing content:} 
This section extends the data-parameter correspondence to the \emph{spatial composition} of multiple sources. Distinct from Section~\ref{subsec:low-rank} (structural constraints via subspace projection) and Section~\ref{subsec:continual} (temporal constraints via trajectory regularization), we examine \emph{cross-sectional fusion}: how to optimally interpolate between distinct data distributions or model parameters through geodesic mixing on the statistical manifold. While prior sections focus on single-task adaptation or sequential learning, here we address the simultaneous integration of $m \geq 2$ heterogeneous data sources or pre-trained model checkpoints.

Fig.~\ref{fig:mixing} illustrates the composition correspondence for integrating multiple heterogeneous sources on the statistical manifold $\mathcal{M}$. \textbf{Left ($\boldsymbol{\eta}$-space)}: Multiple data distributions $\{\mathcal{D}_i\}_{i=1}^m$ are mixed via barycentric interpolation with coefficients $\lambda \in \Delta^{m-1}$, constructing the Bregman barycenter in the expectation parameter space that minimizes $\sum_i \lambda_i D_\phi(\boldsymbol{\eta} \| \boldsymbol{\eta}_i)$. \textbf{Right ($\theta$-space)}: Corresponding model parameters $\{\theta_i^*\}_{i=1}^m$ are merged via affine combination with coefficients $\alpha \in \Delta^{m-1}$, approaching the Bregman barycenter in the natural parameter space with respect to $D_\psi$. \textbf{Bregman Duality}: The optimal mixing coefficients $\lambda$ and merging coefficients $\alpha$ are related through the Legendre duality between $D_\phi$ and $D_\psi$, unifying data mixing and model merging as dual geodesic interpolations under the Fisher-Rao metric $g_{ij}(\theta)$.

\subsubsection*{Problem Setup}

Consider $m$ source data distributions $\mathcal{D}_1, \mathcal{D}_2, \dots, \mathcal{D}_m$ (e.g., domain-specific corpora or task-specific datasets) and the corresponding fine-tuned model parameters $\theta_1^*, \theta_2^*, \dots, \theta_m^*$ obtained from a common pre-trained initialization $\theta_0$. We define two complementary composition operations:

\paragraph*{Data-Side: Distribution Mixing} 
Find mixture coefficients $\lambda \in \Delta^{m-1}$ (the probability simplex) constructing the mixed distribution:
\begin{equation}
\mathcal{D}_\lambda = \sum_{i=1}^m \lambda_i \mathcal{D}_i,
\end{equation}
such that fine-tuning $\theta_0$ on $\mathcal{D}_\lambda$ minimizes the validation loss $\mathcal{L}_{\text{val}}(\theta^*(\lambda))$. Recent work on data mixing laws~\cite{ye2024datamixing} and curriculum learning~\cite{zhang2025adaptivemixing} has demonstrated that optimal $\lambda$ significantly outperforms uniform mixing for large language model training.

\paragraph*{Parameter-Side: Model Merging} 
Find merging coefficients $\alpha \in \Delta^{m-1}$ constructing the merged parameters:
\begin{equation}
\theta_\alpha = \sum_{i=1}^m \alpha_i \theta_i^*,
\end{equation}
such that $\theta_\alpha$ achieves minimal validation loss \emph{without additional training}. Modern merging techniques such as TIES~\cite{yadav2024ties}, DARE~\cite{yu2024dare}, and model breadcrumbs~\cite{goddard2024breadcrumbs} extend beyond simple averaging to handle interference between task vectors.

Both problems seek optimal convex combination coefficients, yet they operate on dual spaces of the statistical manifold $\mathcal{M}$.

\subsubsection*{Geometric Duality: Barycentric Interpolation}

Recall that $\eta$-space (data moments) and $\theta$-space (parameters) constitute dual coordinates on $\mathcal{M}$ (Section~\ref{subsec:info-geo}).

\paragraph*{Data Mixing as $\boldsymbol{\eta}$-Space Barycenter}
For exponential family distributions, the mixture $\mathcal{D}_\lambda$ induces an expectation parameter that approximates the weighted average of individual $\boldsymbol{\eta}_i = \mathbb{E}_{\mathcal{D}_i}[\phi(x)]$:
\begin{equation}
\boldsymbol{\eta}_{\text{mix}}(\lambda) = \sum_{i=1}^m \lambda_i \boldsymbol{\eta}_i + O(\|\theta_i^* - \theta_0\|^2).
\end{equation}
Geometrically, this corresponds to the \emph{Bregman barycenter}~\cite{amari2016information} in $\boldsymbol{\eta}$-coordinates, minimizing the weighted sum of Bregman divergences $D_\phi(\boldsymbol{\eta} \| \boldsymbol{\eta}_i) = \phi(\boldsymbol{\eta}) - \phi(\boldsymbol{\eta}_i) - \langle \nabla \phi(\boldsymbol{\eta}_i), \boldsymbol{\eta} - \boldsymbol{\eta}_i \rangle$.

\paragraph*{Model Merging as $\theta$-Space Barycenter}
Parameter merging performs affine combination directly in the natural parameter space:
\begin{equation}
\theta_{\text{merge}}(\alpha) = \theta_0 + \sum_{i=1}^m \alpha_i (\theta_i^* - \theta_0),
\end{equation}
which, under the local linearity assumption, approximates the Bregman barycenter minimizing $\sum_i \alpha_i D_\psi(\theta \| \theta_i^*)$, where $D_\psi(\theta \| \theta_i^*) = \psi(\theta) - \psi(\theta_i^*) - \langle \nabla \psi(\theta_i^*), \theta - \theta_i^* \rangle$.

\paragraph*{Duality of Barycenters}
The two barycenters are dual in the sense of Legendre duality: $\nabla \psi(\theta) = \boldsymbol{\eta}$ and $\nabla \phi(\boldsymbol{\eta}) = \theta$. Consequently, the optimal mixing coefficients $\lambda^*$ and merging coefficients $\alpha^*$ are related through the Fisher-Rao metric $g_{ij}(\theta)$. When the geodesic distance $d(\theta_i^*, \theta_j^*) \ll 1$, Euclidean averaging in $\theta$-space approximates the geodesic mixing of data distributions; however, as models diverge, one must follow the \emph{Fisher-Rao geodesic} rather than the Euclidean chord to maintain exact correspondence.

\subsubsection*{Unified Optimization and Budget Duality}

We formulate the joint optimization under a shared composition budget:
\begin{equation}
\begin{aligned}
\min_{\lambda, \alpha} \quad & \mathcal{L}_{\text{val}}(\theta_{\text{merge}}(\alpha)) \\
\text{s.t.} \quad & D_\phi(\boldsymbol{\eta}_{\text{mix}}(\lambda) \| \boldsymbol{\eta}_0) + D_\psi(\theta_{\text{merge}}(\alpha) \| \theta_0) \leq \epsilon,
\end{aligned}
\label{eq:joint-composition}
\end{equation}
where the constraint couples the Bregman divergences in dual spaces. This reveals a \emph{budget correspondence}: aggressive data mixing (high entropy $H(\lambda)$) permits sharper parameter merging (concentrated $\alpha$), and vice versa.

\paragraph*{Three-Dimensional Framework}
Table~\ref{tab:three-dimensions} situates this composition correspondence alongside the structural (IV-B) and temporal (IV-C) correspondences, completing a three-dimensional characterization of data-parameter duality.

\subsubsection*{Empirical Corroboration from Recent Work}

The dual relationship between data mixing and model merging articulated above has recently received strong empirical support. Yang et al.~\cite{yang2025mix} systematically compared the two strategies for aligning LLMs with helpfulness, honesty, and harmlessness, finding that model merging achieves higher gains than data mixture methods. In the code generation domain, multi-task studies reveal a scale-dependent trade-off: model merging excels at larger scales while data mixing is preferred at smaller scales~\cite{multitask2026code}. Most compellingly, the independent works MergeMix~\cite{wang2026mergemix} and DeMix~\cite{li2026demix} propose using model merging weights as a low-cost proxy for evaluating data mixture quality---a direct algorithmic instantiation of the compositional duality developed here. These empirical investigations, while not explicitly geometric in their formulation, collectively validate the central thesis of this section: data mixing and model merging are not independent design choices but dual operations on a shared statistical manifold, and exploiting this duality yields principled efficiency gains.

\subsubsection*{Empirical Manifestations}

Beyond these targeted studies, broader advances in model merging and data curation implicitly exploit this correspondence:

\paragraph*{Model Soups and Task Arithmetic} 
The observation that linearly averaging fine-tuned model weights (Model Soups~\cite{wortsman2022model}) often outperforms individual models corresponds to finding the $\theta$-space barycenter. Task Arithmetic~\cite{ilharco2023editing} and its extensions~\cite{ortiz2025task} demonstrate that task vectors $\tau_i = \theta_i^* - \theta_0$ can be added/subtracted, which is the Euclidean approximation of geodesic combination. The interference phenomena observed in naive merging (performance degradation when combining distant tasks) precisely arise from violating the local linearity assumption of the $\theta$-$\boldsymbol{\eta}$ correspondence.

\paragraph*{Data Mixing Laws} 
Recent studies on optimal data composition for LLM training~\cite{ye2024datamixing,xie2025doremi} show that domain-weighting $\lambda$ follows power-law scaling with respect to downstream loss. Under the correspondence, this implies that merging coefficients $\alpha$ should follow similar scaling when interpolating between domain-specific adapters.

\paragraph*{Emerging Hybrid Strategies} 
Methods such as ``merge then fine-tune''~\cite{zhang2025mergeft} or ``mix then merge'' explicitly exploit the composition correspondence: using data mixing to regularize the parameter merging process, or using merged models as warm-start for mixed-data training, achieving Pareto improvements unattainable by either modality alone.

\begin{remark}
Unlike the low-rank constraint (Section~\ref{subsec:low-rank}) or continual trajectory (Section~\ref{subsec:continual}), the composition correspondence is inherently \emph{multi-modal} ($m \geq 2$). When $m=2$, the geodesic interpolation reduces to the Fisher-Rao midpoint studied in information geometry~\cite{amari2016information}; for $m>2$, the Bregman barycenter provides the unique projection onto the mixture flat.
\end{remark}

\begin{table*}[h]
\centering
\footnotesize
\setlength{\tabcolsep}{5pt}
\caption{Three Orthogonal Dimensions of Data-Parameter Correspondence}
\label{tab:three-dimensions}
\begin{tabular}{@{}llll@{}}
\toprule
\textbf{Dimension} & \textbf{Subsection} & \textbf{Geometric Object} & \textbf{Budget Variable} \\
\midrule
Structural (Capacity) & IV-B: Low-Rank & Grassmannian $\mathcal{G}(r,d)$ & Rank $r$ / Shots $k$ \\
Temporal (Sequence) & IV-C: Continual & Fisher-Rao trajectory & $\lambda$ (EWC) / $k$ (buffer) \\
Compositional (Fusion) & IV-H: Mixing/Merging & Bregman barycenter & Mixing entropy $H(\lambda)$ \\
\bottomrule
\end{tabular}
\end{table*}

\subsection{Synthesis: The Manifold Perspective on Data-Parameter Duality}
\label{subsec:synthesis}

The correspondences established in Sections~\ref{subsec:jacobian-framework} through~\ref{subsec:composition} are not isolated empirical observations but distinct manifestations of a single geometric structure: the statistical manifold $\mathcal{M}$ equipped with the Fisher--Rao metric $g_{ij}(\theta)$ (Section~\ref{subsec:info-geo}). This section synthesizes these dualities into a unified framework, examines their inter-dimensional couplings, and provides practical guidelines for operation selection.

\paragraph{Unifying geometric principle}
Each duality originates from the Legendre duality between natural parameters $\theta$ (parameter space) and expectation parameters $\boldsymbol{\eta}$ (data/moment space) on $\mathcal{M}$:
\begin{itemize}
    \item \textbf{Jacobian image space (Sec.~\ref{subsec:jacobian-framework})}: local linearization of $g_{ij}(\theta)$ at $(x,\theta_0)$, yielding $\operatorname{Im}(J_x)$ and $\operatorname{Im}(J_\theta)$ as dual tangent subspaces;
    \item \textbf{Low-rank correspondence (Sec.~\ref{subsec:low-rank})}: restriction to rank-$r$ submanifold $\mathcal{M}_r\subset\mathcal{M}$, where ICL and LoRA explore identical Grassmannian $\mathcal{G}(r,d)$;
    \item \textbf{Continual learning (Sec.~\ref{subsec:continual})}: trajectory constraints on $\mathcal{M}$; replay maintains $\boldsymbol{\eta}$-space support while EWC imposes $\theta$-space curvature, both implementing $(\theta-\theta^*_{\text{old}})^\top \boldsymbol{F}(\theta^*_{\text{old}})(\theta-\theta^*_{\text{old}})\leq\epsilon$;
    %\item \textbf{Shapley duality (Sec.~\ref{subsec:token-shapley})}: second-order path-integral of Fisher metric yielding symmetric importance scores;
    \item \textbf{Composition (Sec.~\ref{subsec:composition})}: data mixing as $\boldsymbol{\eta}$-space Bregman barycenter and model merging as $\theta$-space Bregman barycenter, linked by $\nabla_\theta D_\psi(\theta_{\text{merge}}\|\theta_0)=\boldsymbol{\eta}_{\text{mix}}-\boldsymbol{\eta}_0$;
    \item \textbf{Security \& privacy (Sec.~\ref{subsec:conceptual-security}--\ref{subsec:testing-defense})}: adversarial and protective operations as boundary conditions on $\mathcal{M}$, where cooperative enhancement (Sec.~\ref{subsec:conceptual-security}) exploits $\Delta\mathcal{D}$--$\Delta\theta$ coupling and the product constraint $\epsilon_d\cdot\gamma_p\leq C_{\text{budget}}$ (Sec.~\ref{subsec:testing-defense}) arises from Lipschitz geometry.
\end{itemize}

\begin{figure}[t]
\centering
\includegraphics[width=0.99\columnwidth]{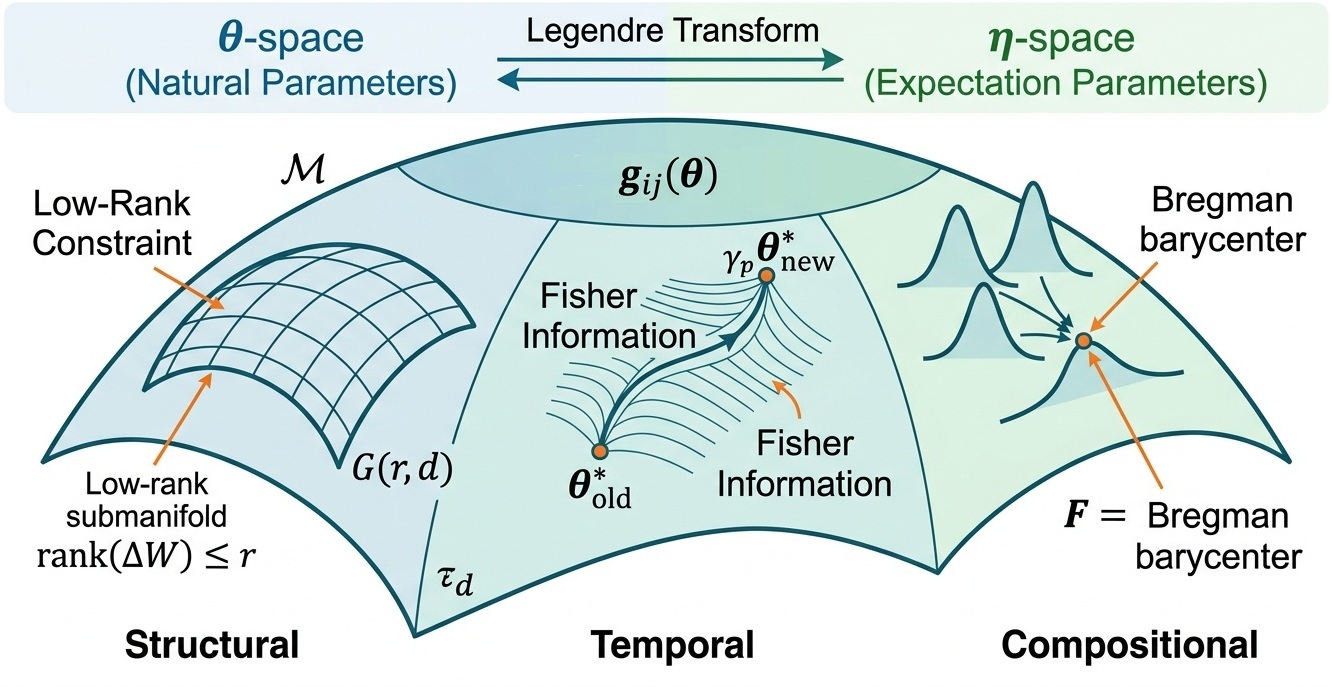}
\caption{Unified view of data-parameter correspondences on statistical manifold $\mathcal{M}$. (\textit{Left}) Structural dimension: low-rank submanifold $\mathcal{M}_r$ (Grassmannian $\mathcal{G}(r,d)$); (\textit{Center}) Temporal dimension: trajectory constraints from $\theta^*_{\text{old}}$ to $\theta^*_{\text{new}}$; (\textit{Right}) Compositional dimension: Bregman barycentric interpolation. Security/privacy constraints manifest as boundaries (dashed lines) on $\mathcal{M}$.}
\label{fig:synthesis-unified}
\end{figure}

Fig.~\ref{fig:synthesis-unified} provides a unified visualization of the data-parameter correspondences on the statistical manifold $\mathcal{M}$. \textbf{Left (Structural)}: The low-rank constraint restricts adaptation to submanifold $\mathcal{M}_r$ (Grassmannian $\mathcal{G}(r,d)$), where ICL and LoRA operate as dual parameterizations of rank-$r$ perturbations. \textbf{Center (Temporal)}: Continual learning imposes trajectory constraints on the Fisher-Rao geodesic from $\theta^*_{\text{old}}$ to $\theta^*_{\text{new}}$, with replay maintaining $\boldsymbol{\eta}$-space support and EWC constraining $\theta$-space curvature via Fisher Information $\boldsymbol{F}$. \textbf{Right (Compositional)}: Data mixing and model merging converge to the Bregman barycenter---data via $\boldsymbol{\eta}$-space interpolation, parameters via $\theta$-space affine combination---connected through Legendre duality. \textbf{Boundary Conditions}: Security and privacy constraints (dashed lines) manifest as exclusion boundaries on $\mathcal{M}$, distinguishing cooperative enhancement (attacks) from cascading attenuation (protection).

\paragraph{Inter-dimensional couplings}
Although Table~\ref{tab:three-dimensions} presents structural, temporal, and compositional correspondences as orthogonal, practical deployment reveals non-trivial interactions. Under strict low-rank budget $r\ll d$ (structural), the privacy cascading effect (Sec.~\ref{subsec:privacy-correspondence}) accelerates because effective dimension $m\approx r$ tightens the mutual information bound $I(\mathcal{D};\theta)\leq m$. Similarly, data mixing (Sec.~\ref{subsec:composition}) alters input distribution geometry, thereby modifying the local Lipschitz constant $\gamma_p$ in testing-time security constraints (Sec.~\ref{subsec:testing-defense}); thus optimal mixing entropy $H(\lambda)$ and contraction condition $\gamma_p<1$ cannot be tuned independently. These couplings suggest that joint optimization across dimensions often outperforms independent tuning.

\paragraph{A decision framework for operation selection}
Table~\ref{tab:decision-matrix} summarizes strategic choices between data-side and parameter-side operations under key constraints. The selection depends primarily on stage differentiability (Stage~1 vs.~3), budget type (sample vs. rank vs. privacy), and security requirements.

\begin{table*}[t]
\caption{Decision Matrix for Data-Centric vs. Parameter-Centric Operations}
\label{tab:decision-matrix}
\centering
\resizebox{\textwidth}{!}{
\begin{tabular}{@{}lp{4.2cm}p{4.2cm}p{5.5cm}@{}}
\toprule
\textbf{Constraint} & \textbf{Data-Side Preferred} & \textbf{Parameter-Side Preferred} & \textbf{Recommended Joint Strategy} \\
\midrule
Stage~3 (Inference) & ICL, Prompt eng., RAG & Dynamic quant., Adapter gating, Early exit & Hybrid: ICL for subspace probe $\to$ dynamic LoRA routing \\
\midrule
Low-Rank Budget ($r\ll d$) & $k$-shot ICL ($k\approx r$) & LoRA (rank $r$) & Coarse-to-fine: ICL locate $\to$ LoRA refine \\
\midrule
Continual Learning & Experience replay (large buffer $k$) & EWC (high $\lambda$) & Small buffer + moderate $\lambda$ (budget $\lambda\cdot k\approx\text{const}$) \\
\midrule
Security (Multi-turn $T\gg1$) & Input sanitization ($\epsilon_d\downarrow$) & Sensitivity control ($\gamma_p<1$) & Product constraint: $\epsilon_d\cdot\gamma_p\leq C_{\text{budget}}$ \\
\midrule
Privacy Budget & Sufficient stats $T(\mathcal{D})$ ($\rho\downarrow$) & Gradient compression ($\rho$) & Cascading: $\epsilon_{\text{eff}}\leq\epsilon_d+\rho\cdot\epsilon_p$ \\
\bottomrule
\end{tabular}%
}
\end{table*}

\paragraph{Boundary conditions and breakdown of duality}
These correspondences assume local validity on $\mathcal{M}$. Three conditions induce breakdown:
\begin{enumerate}
    \item \textbf{Non-local divergence}: When $\|\theta_i^*-\theta_0\|$ is large (e.g., after extensive continual learning without replay), the Legendre transformation between $\boldsymbol{\eta}$ and $\theta$ becomes non-invertible, invalidating linearized approximations in Sections~\ref{subsec:jacobian-framework} and~\ref{subsec:low-rank}.
    \item \textbf{Non-differentiable operations}: Hard quantization (INT2) and discrete routing (MoE) introduce non-smooth boundaries on $\mathcal{M}$ where the Fisher metric is undefined.
    \item \textbf{Adversarial regime shift}: Under $\mathcal{D}_{\text{train}}\neq\mathcal{D}_{\text{test}}$, Jacobian image spaces $\operatorname{Im}(J_x)$ and $\operatorname{Im}(J_\theta)$ may misalign, breaking functional equivalence (Eq.~\eqref{eq:jacobian-decomp}).
\end{enumerate}
These limitations delineate the frontier for extending the correspondence, rather than undermining the framework.

\paragraph{Transition to future directions}
The synthesis above reveals that data and parameters constitute two coordinate systems on the same statistical manifold. The operational implications---that data pruning and parameter sparsification, in-context learning and low-rank adaptation, or data mixing and model merging are geometrically dual---suggest new opportunities for cross-domain optimization. The following conclusion articulates specific directions for extending these correspondences beyond the present preliminary investigation.

\section{Algorithmic Realization and Computational Feasibility}
\label{sec:algorithms}

While the data-parameter correspondences in Sections~\ref{sec:established}--\ref{sec:novel} are grounded in geometric abstraction, their utility depends on tractability at billion-parameter scales. This section demonstrates that all constructs---the Gradient Interaction Matrix~$\boldsymbol{M}$, Fisher-Rao metric~$g_{ij}(\theta)$, and Lipschitz constraints~$\gamma_p$---admit efficient implementations via standard automatic differentiation primitives. We present streaming algorithms that reduce memory from prohibitive $O(N\cdot K)$ or $O(K^2)$ to $O(K)$, rendering the framework deployable.

\subsection{Streaming Computation of the Gradient Interaction Matrix}
\label{subsec:gradient_interaction_algo}

Recall from Eq.~(\ref{eq:gradient_interaction_matrix}) that $M_{n,k} = g_{n,k} \cdot v_k$. Rather than materializing $\mathbf{M}\in\mathbb{R}^{N\times K}$, we compute the sufficient statistics---data utilities $\{\phi_n^d\}$ and parameter importances $\{\phi_k^p\}$---via Jacobian-Vector Products (JVP)~\cite{frostig2018compiling}.

\begin{algorithm}[h]
\caption{Streaming Selection Score Computation}
\label{alg:streaming}
\begin{algorithmic}[1]
\STATE \textbf{Input:} Model $f_\theta$, training set $\{(x_n,y_n)\}_{n=1}^N$, validation gradient $\mathbf{v}=\nabla_\theta\mathcal{L}_{\text{val}}(\bar{\theta})$, budgets $(b,\rho)$
\STATE \textbf{Output:} Selected indices $\mathcal{S}_d$, parameter mask $\mathcal{S}_p$
\STATE Initialize $\boldsymbol{\phi}^d \in \mathbb{R}^N$, $\boldsymbol{\phi}^p \in \mathbb{R}^K$ as zero vectors
\FOR{$n = 1$ \textbf{to} $N$}
    \STATE Compute $\mathbf{g}_n = \nabla_\theta \ell(f_{\bar{\theta}}; x_n, y_n)$ via JVP
    \STATE $\phi_n^d \leftarrow \sum_{k=1}^K g_{n,k} \cdot v_k$ \COMMENT{Row-wise: $\phi_n^d = \sum_k M_{n,k}$ as in Eq.~(6a)}
    \STATE $\boldsymbol{\phi}^p \leftarrow \boldsymbol{\phi}^p + (\mathbf{g}_n \odot \mathbf{v})$ \COMMENT{Column-wise: accumulates $\phi_k^p = \sum_n M_{n,k}$ as in Eq.~(6b)}
\ENDFOR
\STATE $\mathcal{S}_d \leftarrow \text{TopK}(\boldsymbol{\phi}^d, b)$
\STATE $\mathcal{S}_p \leftarrow \text{TopK}(\boldsymbol{\phi}^p, \lfloor\rho K\rfloor)$
\STATE \textbf{return} $\mathcal{S}_d, \mathcal{S}_p$
\end{algorithmic}
\end{algorithm}

\begin{theorem}[Streaming Complexity]
\label{thm:streaming}
Algorithm~\ref{alg:streaming} computes the unified selection scores with:
\begin{itemize}
    \item \textbf{Time:} $O(N\cdot T_{\text{grad}} + K\log K)$, where $T_{\text{grad}}$ is single-sample gradient cost;
    \item \textbf{Memory:} $O(K)$, independent of $N$;
    \item \textbf{Passes:} Single pass over training data.
\end{itemize}
\end{theorem}

\begin{proof}
By linearity of differentiation, $\phi_n^d$ and $\boldsymbol{\phi}^p$ are additive over samples. Each iteration computes $\mathbf{g}_n$ via forward-mode AD (JVP) without storing intermediate activations for other samples. The TopK operations dominate the $K\log K$ term.
\end{proof}

\subsection{Fisher-Rao Metric Estimation via Sketching}
\label{subsec:fisher_estimation}

The diagonal Fisher Information $\text{diag}(\mathbf{F})$ required for EWC (Sec.~\ref{subsec:continual}) and geometric volume computation (Sec.~\ref{subsec:geometric-correspondence}) cannot be stored explicitly. We employ Hutchinson's stochastic estimator~\cite{hutchinson1989stochastic}.

\begin{algorithm}[h]
\caption{Diagonal Fisher Estimation}
\label{alg:hutchinson}
\begin{algorithmic}[1]
\STATE \textbf{Input:} Model $f_\theta$, data distribution $\mathcal{D}$, sketch size $m$
\STATE \textbf{Output:} Diagonal estimate $\hat{\mathbf{f}} \in \mathbb{R}^K$
\STATE Initialize $\hat{\mathbf{f}} \leftarrow \mathbf{0}$
\FOR{$i = 1$ \textbf{to} $m$}
    \STATE Sample $\mathbf{z} \in \{-1, +1\}^K$ uniformly (Rademacher)
    \STATE Compute $\mathbf{J}\mathbf{z}$ via JVP \COMMENT{$\mathbf{J}$ is Jacobian of log-likelihood}
    \STATE $\hat{\mathbf{f}} \leftarrow \hat{\mathbf{f}} + (\mathbf{J}\mathbf{z})^2 \odot \mathbf{z}$ \COMMENT{Element-wise square and multiply}
\ENDFOR
\STATE \textbf{return} $\hat{\mathbf{f}} / m$
\end{algorithmic}
\end{algorithm}

\begin{proposition}[Sketch Variance]
\label{prop:fisher_variance}
For each coordinate $k$, the estimator satisfies $\mathbb{V}[\hat{f}_k] = O(1/m)$. Thus $m = O(1/\epsilon^2)$ suffices to estimate high-curvature directions (large $F_{kk}$) required for EWC regularization.
\end{proposition}

\subsection{Grassmannian Subspace Operations}
\label{subsec:grassmannian_ops}

The low-rank correspondence (Sec.~\ref{subsec:low-rank}) requires comparing the subspace induced by $k$ demonstrations (data-side) with the column space of $\mathbf{A}\in\mathbb{R}^{K\times r}$ from LoRA (parameter-side).

\paragraph{ICL Subspace Extraction.}
For $k$ demonstrations inducing implicit gradient directions $\{\mathbf{u}_i\}_{i=1}^k \subset \mathbb{R}^K$ (via forward-mode analysis of attention patterns~\cite{vonoswald2023transformers}), form matrix $\mathbf{U} = [\mathbf{u}_1, \ldots, \mathbf{u}_k]^\top \in \mathbb{R}^{k\times K}$. Compute the thin SVD:
\[
\mathbf{U}_{\text{ICL}} = \text{right-singular-vectors}(\mathbf{U}, r) \in \mathbb{R}^{K\times r}
\]
yielding an orthonormal basis for the $r$-dimensional subspace $\mathcal{M}_r \subseteq \mathcal{G}(r,d)$.

\paragraph{Alignment Verification.}
Given LoRA matrices $\mathbf{A}\in\mathbb{R}^{K\times r}, \mathbf{B}\in\mathbb{R}^{r\times K}$, extract $\mathbf{Q}_{\text{LoRA}} = \text{QR}(\mathbf{A}) \in \mathbb{R}^{K\times r}$. The principal angles $\{\theta_i\}_{i=1}^r$ between $\text{Im}(\mathbf{U}_{\text{ICL}})$ and $\text{Im}(\mathbf{Q}_{\text{LoRA}})$ satisfy:
\[
\cos(\theta_i) = \sigma_i(\mathbf{Q}_{\text{LoRA}}^\top \mathbf{U}_{\text{ICL}})
\]
where $\sigma_i(\cdot)$ denotes the $i$-th singular value. Small angles empirically validate the $k \approx r$ equivalence (Table~\ref{tab:low-rank}).

\paragraph{Complexity.}
Both SVD and QR cost $O(K \cdot r^2)$ time and $O(K \cdot r)$ memory, feasible since $r \ll K$ (typically $r \in \{8,16,64\}$ per Table~\ref{tab:notation}).

\subsection{Sketched EWC with Replay}
\label{subsec:sketched_ewc}

The continual learning correspondence $k \cdot \lambda \approx \text{const}$ (Sec.~\ref{subsec:continual}) manifests algorithmically through the interplay between buffer size and Fisher estimation variance.

\begin{algorithm}[h]
\caption{Replay-Based Fisher Regularization}
\label{alg:sketched_ewc}
\begin{algorithmic}[1]
\STATE \textbf{Input:} New task data $\mathcal{D}_{\text{new}}$, replay buffer $\mathcal{B}$ ($|\mathcal{B}|=k$), old parameters $\theta_{\text{old}}$, coefficient $\lambda$
\STATE \textbf{Output:} Regularized loss $\mathcal{L}_{\text{CL}}$
\STATE $\hat{\mathbf{F}} \leftarrow \text{Algorithm~\ref{alg:hutchinson}}(f_\theta, \mathcal{B}, m)$ \COMMENT{Empirical Fisher from $k$ samples}
\STATE $\mathcal{L}_{\text{new}} \leftarrow \frac{1}{|\mathcal{D}_{\text{new}}|}\sum_{(x,y)\in\mathcal{D}_{\text{new}}} \ell(f_\theta(x), y)$
\STATE $\mathcal{R}_{\text{EWC}} \leftarrow \frac{\lambda}{2} \sum_{j=1}^K \hat{F}_{jj} (\theta_j - \theta_{\text{old},j})^2$
\STATE \textbf{return} $\mathcal{L}_{\text{new}} + \mathcal{R}_{\text{EWC}}$
\end{algorithmic}
\end{algorithm}

\begin{remark}[Buffer-Regularization Duality]
Algorithm~\ref{alg:sketched_ewc} implements a hybrid of replay and EWC. When $k$ is small, $\hat{\mathbf{F}}$ has high variance $O(1/k)$, necessitating larger $\lambda$ to constrain $\theta$ within the high-curvature trust region. This algorithmic constraint mirrors the theoretical trade-off $k \cdot \lambda \approx \text{const}$: reducing buffer size (memory budget) requires increasing regularization strength (computational budget) to maintain equivalent trajectory constraints on $\mathcal{M}$.
\end{remark}

\subsection{Power Iteration for Lipschitz Constants}
\label{subsec:lipschitz_estimation}

Verifying the safety constraint $\epsilon_d \cdot \gamma_p \leq C_{\text{budget}}$ (Sec.~\ref{subsec:testing-defense}) requires estimating the local Lipschitz constant $\gamma_p = \|\nabla_e f_{\theta^*}\|_{\text{op}}$ without materializing the Jacobian.

\begin{algorithm}[h]
\caption{Local Lipschitz Estimation}
\label{alg:power_iteration}
\begin{algorithmic}[1]
\STATE \textbf{Input:} Frozen model $f_{\theta^*}$, input embedding $e$, iterations $T$
\STATE \textbf{Output:} Lipschitz estimate $\hat{\gamma}_p$
\STATE Initialize $\mathbf{v} \sim \mathcal{N}(0, \mathbf{I})$, normalize $\mathbf{v} \leftarrow \mathbf{v}/\|\mathbf{v}\|_2$
\FOR{$t = 1$ \textbf{to} $T$}
    \STATE $\mathbf{v}_{\text{new}} \leftarrow \text{VJP}(\mathbf{v}; f_{\theta^*}, e)$ \COMMENT{Vector-Jacobian Product: $\mathbf{v}^\top \nabla_e f$}
    \STATE $\mathbf{v} \leftarrow \mathbf{v}_{\text{new}} / \|\mathbf{v}_{\text{new}}\|_2$
\ENDFOR
\STATE $\hat{\gamma}_p \leftarrow \|\text{JVP}(\mathbf{v}; f_{\theta^*}, e)\|_2$ \COMMENT{Rayleigh quotient for top singular value}
\STATE \textbf{return} $\hat{\gamma}_p$
\end{algorithmic}
\end{algorithm}

Algorithm~\ref{alg:power_iteration} converges to the largest singular value of $\nabla_e f_{\theta^*}$ with rate dependent on the spectral gap~\cite{golub2013matrix}. The computation requires $O(T \cdot T_{\text{forward}})$ time and $O(d)$ memory, where $d$ is embedding dimension.

\subsection{Summary of Computational Feasibility}

Table~\ref{tab:complexity} summarizes the complexity gains. All operations reduce to standard JVP/VJP primitives available in PyTorch/JAX.

\begin{table}[h]
\caption{Algorithmic Realization of Theoretical Constructs}
\label{tab:complexity}
\scriptsize
\centering
\begin{tabular}{lccc}
\toprule
\textbf{Construct} & \textbf{Naive Cost} & \textbf{Algorithm} & \textbf{Efficient Cost} \\
\midrule
Grad. Interaction $\mathbf{M}$ & $O(NK)$ memory & Alg.~\ref{alg:streaming} & $O(K)$ memory \\
Fisher Diagonal $F_{kk}$ & $O(K^2)$ Hessian & Alg.~\ref{alg:hutchinson} & $O(mK)$ sketch \\
ICL Subspace & Backprop required & Sec.~\ref{subsec:grassmannian_ops} & Forward only \\
EWC Regularization & Store full $\mathbf{F}$ & Alg.~\ref{alg:sketched_ewc} & $O(kK)$ from buffer \\
Lipschitz $\gamma_p$ & SVD of Jacobian & Alg.~\ref{alg:power_iteration} & $O(Td)$ power iter. \\
\bottomrule
\end{tabular}
\end{table}

\begin{remark}
The algorithms above demonstrate that the data-parameter correspondences are not abstract geometric curiosities but \emph{computationally tractable procedures}. The unified framework permits joint optimization over data and parameter modalities using standard deep learning primitives.
\end{remark}

\section{Conclusion}
\label{sec:conclusion}

This paper establishes a systematic correspondence between data-centric and parameter-centric operations in LLMs, grounded in the Fisher--Rao metric $g_{ij}(\theta)$ and Legendre duality between natural ($\theta$) and expectation ($\eta$) parameters. We formalize three fundamental correspondences spanning the model lifecycle: \textbf{(1)~Geometric}---data pruning and parameter sparsification reduce manifold volume via dual coordinate constraints (Sec.~III-A); %\textbf{(2)~Optimization}---the Gradient Interaction Matrix $\mathbf{M}\in\mathbb{R}^{N\times K}$ exposes symmetric bilevel structures between data selection and parameter masking (Sec.~II-C and III-B); 
\textbf{(2)~Low-rank}---in-context learning and LoRA adaptation explore identical subspaces on the Grassmannian $\mathcal{G}(r,d)$, with $k$-shot samples geometrically equivalent to rank-$r$ updates (Sec.~IV-B); and \textbf{(3)~Security-privacy}---adversarial attacks exhibit cooperative amplification across the data-parameter boundary, whereas protective mechanisms follow cascading attenuation governed by mutual information $I(\mathcal{D};\theta)$ (Sec.~IV-D, Sec.~IV-E and IV-F).

These correspondences manifest across three orthogonal dimensions: \textbf{structural} (low-rank constraints via $\mathcal{G}(r,d)$), \textbf{temporal} (trajectory constraints via experience replay versus elastic weight consolidation, Sec.~IV-C), and \textbf{compositional} (Bregman barycentric interpolation for data mixing and model merging, Sec.~IV-H). Furthermore, the testing-time security constraint $\epsilon_d\cdot\gamma_p\leq C_{\text{budget}}$ demonstrates that input sanitization and Lipschitz regularization are multiplicatively coupled, with $\gamma_p<1$ being necessary and sufficient for long-form safety (Sec.~IV-F).

As a preliminary investigation, this work does not claim a unified theory covering all LLM lifecycles. Limitations include: (i) the Jacobian image space framework (Sec.~IV-A) is local and first-order; (ii) the product constraint assumes Lipschitz continuity that may not hold globally for hard quantization or discrete routing; (iii) privacy cascading bounds rely on Fisher approximations that become noisy in deep networks; and (iv) non-differentiable operations introduce non-smooth boundaries where the Fisher metric is undefined. These limitations delineate the frontier for future inquiry.

Promising directions include: \textbf{energy-based duality}, linking data distillation with parameter pruning via free-energy minimization on $\mathcal{M}$; \textbf{knowledge editing correspondences}, where data-side counterfactual augmentation and parameter-side model editing share rank-one update structures; \textbf{MoE routing as simultaneous selection}, treating expert routing as joint data--parameter sparsification; and \textbf{full information-geometric unification across stages}, potentially via a time-dependent Fisher metric interpolating between training and inference. We view this work not as a closed theory but as a roadmap for rethinking LLM engineering through the lens of manifold duality.

\bibliography{ReferenceFormat}
\bibliographystyle{IEEEtran}

\vfill

\end{document}